\documentclass[twoside]{article}
\usepackage[accepted]{aistats2022}

\usepackage[round]{natbib}

\usepackage{algorithm}
\usepackage{algorithmicx}
\usepackage[noend]{algpseudocode}
\usepackage{amsmath}
\usepackage{amssymb}
\usepackage{amsthm}
\usepackage{bbm}
\usepackage{bm}
\usepackage{caption}
\usepackage{color}
\usepackage{dirtytalk}
\usepackage{dsfont}
\usepackage{enumerate}
\usepackage{graphicx}
\usepackage{listings}
\usepackage{mathtools}
\usepackage{subfigure}
\usepackage{xspace}

\usepackage[usenames,dvipsnames]{xcolor}
\usepackage[bookmarks=false]{hyperref}
\hypersetup{
  pdffitwindow=true,
  pdfstartview={FitH},
  pdfnewwindow=true,
  colorlinks,
  linktocpage=true,
  linkcolor=Green,
  urlcolor=Green,
  citecolor=Green
}
\usepackage[capitalize,noabbrev]{cleveref}

\usepackage[backgroundcolor=white]{todonotes}

\newcommand{\commentout}[1]{}
\newcommand{\junk}[1]{}
\usepackage{tikz}
\usetikzlibrary{bayesnet}
\Crefname{corollary}{Corollary}{Corollaries}
\Crefname{proposition}{Proposition}{Propositions}
\Crefname{theorem}{Theorem}{Theorems}
\Crefname{definition}{Definition}{Definitions}
\Crefname{assumption}{Assumption}{Assumptions}
\Crefname{example}{Example}{Examples}
\Crefname{remark}{Remark}{Remarks}
\Crefname{setting}{Setting}{Settings}
\Crefname{lemma}{Lemma}{Lemmas}

\usepackage{thmtools}
\declaretheorem[name=Theorem,refname={Theorem,Theorems},Refname={Theorem,Theorems}]{theorem}
\declaretheorem[name=Lemma,refname={Lemma,Lemmas},Refname={Lemma,Lemmas},sibling=theorem]{lemma}

\declaretheorem[name=Assumption,refname={Assumption,Assumptions},Refname={Assumption,Assumptions}]{assumption}

\newcommand{\cA}{\mathcal{A}}

\newcommand{\cC}{\mathcal{C}}

\newcommand{\cL}{\mathcal{L}}

\newcommand{\cN}{\mathcal{N}}

\newcommand{\cS}{\mathcal{S}}

\newcommand{\cV}{\mathcal{V}}

\newcommand{\realset}{\mathbb{R}}

\newcommand{\diag}[1]{\mathrm{diag}\left(#1\right)}

\newcommand{\E}[2]{\mathbb{E}_{#1} \left[#2\right]}
\newcommand{\condE}[2]{\mathbb{E} \left[#1 \,\middle|\, #2\right]}

\newcommand{\prob}[1]{\mathbb{P} \left(#1\right)}
\newcommand{\condprob}[2]{\mathbb{P} \left(#1 \,\middle|\, #2\right)}

\newcommand{\var}[1]{\mathrm{var} \left[#1\right]}
\newcommand{\condvar}[2]{\mathrm{var} \left[#1 \,\middle|\, #2\right]}

\newcommand{\condcov}[2]{\mathrm{cov} \left[#1 \,\middle|\, #2\right]}

\newcommand{\abs}[1]{\left|#1\right|}

\newcommand*\dif{\mathop{}\!\mathrm{d}}

\newcommand{\I}[1]{\mathds{1} \! \left\{#1\right\}}

\newcommand{\maxnorm}[1]{\|#1\|_\infty}

\newcommand{\normw}[2]{\|#1\|_{#2}}

\newcommand{\set}[1]{\left\{#1\right\}}

\newcommand{\T}{^\top}

\DeclareMathOperator*{\argmax}{arg\,max\,}

\let\det\relax
\DeclareMathOperator{\det}{det}

\let\trace\relax
\DeclareMathOperator{\trace}{tr}
\mathchardef\mhyphen="2D

\newcommand{\hierts}{\ensuremath{\tt HierTS}\xspace}
\newcommand{\lints}{\ensuremath{\tt LinTS}\xspace}
\newcommand{\linucb}{\ensuremath{\tt LinUCB}\xspace}
\newcommand{\oraclets}{\ensuremath{\tt OracleTS}\xspace}
\newcommand{\ts}{\ensuremath{\tt TS}\xspace}

\def\Bregret{\mathcal{BR}}

\begin{document}

\twocolumn[

\aistatstitle{Hierarchical Bayesian Bandits}

\aistatsauthor{Joey Hong \And Branislav Kveton \And Manzil Zaheer \And Mohammad Ghavamzadeh}

\aistatsaddress{UC Berkeley$^*$ \And Amazon$^*$ \And Google DeepMind \And Google Research}]

\begin{abstract}
Meta-, multi-task, and federated learning can be all viewed as solving similar tasks, drawn from a distribution that reflects task similarities. We provide a unified view of all these problems, as learning to act in a \emph{hierarchical Bayesian bandit}. We propose and analyze a natural hierarchical Thompson sampling algorithm (\hierts) for this class of problems. Our regret bounds hold for many variants of the problems, including when the tasks are solved sequentially or in parallel; and show that the regret decreases with a more informative prior. Our proofs rely on a novel total variance decomposition that can be applied beyond our models. Our theory is complemented by experiments, which show that the hierarchy helps with knowledge sharing among the tasks. This confirms that hierarchical Bayesian bandits are a universal and statistically-efficient tool for learning to act with similar bandit tasks.
\end{abstract}

\section{INTRODUCTION}
\label{sec:introduction}

A \emph{stochastic bandit} \citep{lai85asymptotically,auer02finitetime,lattimore19bandit} is an online learning problem where a \emph{learning agent} sequentially interacts with an environment over $n$ rounds. In each round, the agent takes an \emph{action} and receives a \emph{stochastic reward}. The agent aims to maximize its expected cumulative reward over $n$ rounds. It does not know the mean rewards of the actions \emph{a priori}, and must learn them by taking the actions. This induces the \emph{exploration-exploitation dilema}: \emph{explore}, and learn more about an action; or \emph{exploit}, and take the action with the highest estimated reward. In online advertising, an action could be showing an advertisement and its reward could be an indicator of a click.

\renewcommand{\thefootnote}{\fnsymbol{footnote}}
\footnotetext[1]{The work started while being at Google Research.}
\renewcommand{\thefootnote}{\arabic{footnote}}

More statistically-efficient exploration is the primary topic of bandit papers. This is attained by leveraging the structure of the problem, such as the form of the reward distribution \citep{garivier11klucb}, prior distribution over model parameters \citep{thompson33likelihood,agrawal12analysis,chapelle11empirical,russo18tutorial}, conditioning on known feature vectors \citep{dani08stochastic,abbasi-yadkori11improved,agrawal13thompson}, or modeling the process by which the total reward arises \citep{radlinski08learning,kveton15cascading,gai12combinatorial,chen16combinatorial,kveton15tight}. In this work, we solve multiple similar bandit tasks, and each task teaches the agent how to solve other tasks more efficiently. 

We formulate the problem of learning to solve similar bandit tasks as regret minimization in a \emph{hierarchical Bayesian model} \citep{gelman13bayesian}. Each task is parameterized by a \emph{task parameter}, which is sampled i.i.d.\ from a distribution parameterized by a \emph{hyper-parameter}. The parameters are unknown and this relates all tasks, in the sense that each task teaches the agent about any other task. We derive Bayes regret bounds that reflect the structure of the problem and show that the price for learning the hyper-parameter is low. Our derivations use a novel \emph{total variance decomposition}, which decomposes the parameter uncertainty into per-task uncertainty conditioned on knowing the hyper-parameter and hyper-parameter uncertainty. After that, we individually bound each uncertainty source by elliptical lemmas \citep{dani08stochastic,abbasi-yadkori11improved}. Our approach can be exactly implemented and analyzed in hierarchical multi-armed and linear bandits with Gaussian rewards, but can be extended to other graphical model structures.

We build on numerous prior works that study a similar structure, under the names of collaborating filtering bandits \citep{gentile14online,kawale15efficient,li16collaborative}, bandit meta-learning and multi-task learning \citep{azar13sequential,deshmukh17multitask,bastani19meta,cella20metalearning,kveton21metathompson,moradipari21parameter}, and representation learning \citep{yang21impact}. Despite it, we make major novel contributions, both in terms of a more general setting and analysis techniques. Our setting relaxes the assumptions that the tasks are solved in a sequence and that exactly one task is solved per round. Moreover, while the design of our posterior sampling algorithm is standard, we make novel contributions in its analysis. In the sequential setting (\cref{sec:sequential regret}), we derive a Bayes regret bound by decomposing the posterior covariance, which is an alternative to prior derivations based on filtered mutual information \citep{russo16information,lu19informationtheoretic}. This technique is general, simple, and yields tighter regret bounds because it avoids marginal task parameter covariance; and so is of a broad interest. In the concurrent setting (\cref{sec:concurrent regret}), we bound the additional regret due to not updating the posterior after each interaction. This is non-trivial and a major departure from other bandit analyses. Our Bayes regret bound for this setting is the first of its kind.

The paper is organized as follows. In \cref{sec:setting}, we formalize our setting of \emph{hierarchical Bayesian bandits}. In \cref{sec:algorithm}, we introduce a natural Thompson sampling algorithm (\hierts) for solving it. In \cref{sec:models}, we instantiate it in hierarchical Gaussian models. In \cref{sec:key ideas}, we review key ideas in our regret analyses, including a novel total covariance decomposition that allows us to analyze posteriors in hierarchical models. In \cref{sec:regret bounds}, we prove Bayes regret bounds for \hierts in sequential and concurrent settings. Finally, in \cref{sec:experiments}, we evaluate \hierts empirically to confirm our theoretical results.

\section{SETTING}
\label{sec:setting}

We use the following notation. Random variables are capitalized, except for Greek letters like $\theta$ and $\mu$. For any positive integer $n$, we define $[n] = \set{1, \dots, n}$. The indicator function is denoted by $\I{\cdot}$. The $i$-th entry of vector $v$ is $v_i$. If the vector is already indexed, such as $v_j$, we write $v_{j, i}$. A matrix with diagonal entries $v$ is $\diag{v}$. For any matrix $M \in \realset^{d \times d}$, the maximum eigenvalue is $\lambda_1(M)$ and the minimum is $\lambda_d(M)$. The big O notation up to logarithmic factors is $\tilde{O}$.

Now we present our setting for solving similar bandit tasks. Each task is a \emph{bandit instance} with actions $a \in \cA$, where $\cA$ denotes an \emph{action set}. Rewards of actions are generated by \emph{reward distribution} $P(\cdot \mid a; \theta)$, where $\theta \in \Theta$ is an unknown parameter shared by all actions. We assume that the rewards are $\sigma^2$-sub-Gaussian and denote by $r(a; \theta) = \E{Y \sim P(\cdot \mid a; \theta)}{Y}$ the mean reward of action $a$ under $\theta$. The learning agent interacts with $m$ tasks. In a recommender system, each task could be an individual user. The task $s \in [m]$ is parameterized by a \emph{task parameter} $\theta_{s, *} \in \Theta$, which is sampled i.i.d.\ from a \emph{task prior distribution} $\theta_{s, *} \sim P(\cdot \mid \mu_*)$; which is parameterized by an unknown \emph{hyper-parameter} $\mu_*$.

The agent acts at discrete decision points, which are integers and we call them \emph{rounds}. At round $t \geq 1$, the agent is asked to act in a set of tasks $\cS_t \subseteq [m]$. It takes actions $A_t = (A_{s, t})_{s \in \cS_t}$, where $A_{s, t} \in \cA$ is the action in task $s$; and receives rewards $Y_t = (Y_{s, t})_{s \in \cS_t} \in \realset^{|\cS_t|}$, where $Y_{s, t} \sim P(\cdot \mid A_{s, t}; \theta_{s, *})$ is a stochastic reward for taking action $A_{s, t}$ in task $s$. The rewards are drawn i.i.d.\ from their respective distributions. The set $\cS_t$ can depend arbitrarily on the history. The assumption that the action set $\cA$ is the same across all tasks and rounds is only to simplify exposition.

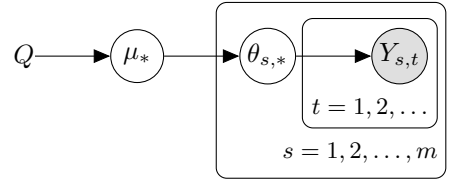
\begin{figure}[t]
  \centering
  \begin{tikzpicture}

  % Define nodes
  \node[obs] (y) {$Y_{s, t}$};
  \node[latent, left=of y] (theta) {$\theta_{s, *}$};
  \node[latent, left=of theta]  (mu) {$\mu_*$};
  \node[const, left=of mu] (q) {$Q$};
  % Connect the nodes
  \edge {q} {mu} ;
  \edge {mu} {theta} ; %
  \edge {theta} {y} ; %

  % Plates
  \plate {y} {(y)} {$t = 1, 2, \hdots$} ;
  \plate[inner sep=.2cm, xshift=-.1cm] {theta} {(theta)(y)} {$s = 1, 2, \hdots, m$} ;

\end{tikzpicture}
  \caption{Graphical model of our hierarchical Bayesian bandit.}
  \label{fig:setting}
\end{figure}

In \emph{hierarchical Bayesian bandits}, the hyper-parameter $\mu_*$ is initially sampled from a \emph{hyper-prior} $Q$ known by the learning agent. Our full model is given by
\begin{align*}
  \mu_* & \sim Q\,, \\
  \theta_{s, *} \mid \mu_*
  & \sim P(\cdot \mid \mu_*)\,,
  & \forall s \in [m]\,, \\
  Y_{s, t} \mid A_{s, t}, \theta_{s, *}
  & \sim P(\cdot \mid A_{s, t}; \theta_{s, *})\,,
  & \forall t \geq 1, \, s \in \cS_t\,,
\end{align*}
and also visualized in \cref{fig:setting}. Note that $P$ denotes both the task prior and reward distribution; but they can be distinguished based on their parameters. Our setting is an instance of a hierarchical Bayesian model commonly used in supervised learning \citep{lindley72bayes,zhang17survey}, and has been studied in bandits in special cases where tasks appear sequentially \citep{kveton21metathompson,basu21noregrets}. 

Our learning agent interacts with each of $m$ tasks for at most $n$ times. So the total number of rounds varies, as it depends on the number of tasks that the agent interacts with simultaneously in each round. However, the maximum number of interactions is $m n$. The goal is to minimize the \emph{Bayes regret} \citep{russo14learning} defined as
\begin{align*}
  \Bregret(m, n)
  = \E{}{\sum_{t \geq 1} \sum_{s \in \cS_t}
  r(A_{s, *}; \theta_{s, *}) - r(A_{s, t}; \theta_{s, *})}\,,
\end{align*}
where $A_{s, *} = \argmax_{a \in \cA} r(a; \theta_{s, *})$ is the optimal action in task $s$. The expectation in $\Bregret(m, n)$ is over $\mu_*$, $\theta_{s, *}$ for $s \in [m]$, and actions $A_{s,t}$ of the agent. While weaker than a traditional frequentist regret, the Bayes regret is a practical performance metric, as we are often interested in an average performance \citep{hong20latent,kveton21metathompson}. For example, when recommending to a group of users, it is natural to optimize over the whole population rather than an individual. Our definition of $\Bregret(m, n)$ is also dictated by the fact that $m$ and $n$ are the primary quantities of interest in our regret analyses. Our goal is to minimize $\Bregret(m, n)$ without knowing $\mu_*$ and $\theta_{s, *}$ a priori.

Since the set of tasks $\cS_t$ can be chosen arbitrarily, our setting is general and subsumes many prior settings. For instance, when the agent interacts with the same task for $n$ rounds before shifting to the next one, and $\cS_t = \{\lceil t / n \rceil\}$, we get a \emph{meta-learning bandit} \citep{kveton21metathompson}. More generally, when the agent interacts with the tasks sequentially, $|\cS_t|$ = 1, our setting can be viewed as multi-task learning where any task helps the agent to solve other tasks. Therefore, we recover a \emph{multi-task bandit} \citep{wan21metadatabased}. Finally, when the agent acts in multiple tasks concurrently, $|\cS_t| > 1$, we recover the setting of \emph{collaborative filtering bandits} \citep{gentile14online,li16collaborative} or more recently \emph{federated bandits} \citep{shi21federated}. Our algorithm and its analysis apply to all of these settings.

\section{ALGORITHM}
\label{sec:algorithm}

\begin{algorithm}[t]
  \caption{Hierarchical Thompson sampling.}
  \label{alg:ts}
  \begin{algorithmic}[1]
    \State \textbf{Input:} Hyper-prior $Q$
    \State Initialize $Q_1 \gets Q$
    \For{$t = 1, 2, \dots$}
      \State Observe tasks $\cS_t \subseteq [m]$
      \State Sample $\mu_t \sim Q_t$
      \For{$s \in \cS_t$}
        \State Compute $P_{s, t}(\theta \mid \mu_t) \propto \cL_{s, t}(\theta) P(\theta \mid \mu_t)$
        \State Sample $\theta_{s, t} \sim P_{s, t}(\cdot \mid \mu_t)$
        \State Take action $A_{s, t} \gets \argmax_{a \in \cA} r(a; \theta_{s, t})$
        \State Observe reward $Y_{s, t}$
      \EndFor
      \State Update $Q_{t + 1}$
    \EndFor
  \end{algorithmic}
\end{algorithm}

We take a Bayesian view and use hierarchical \emph{Thompson sampling (TS)}, which we call \hierts, to solve our problem class. \hierts samples task parameters from their posterior conditioned on history. Specifically, let $H_{s, t} = ((A_{s, \ell}, Y_{s, \ell}))_{\ell < t, \, s \in \cS_\ell}$ denote the history of all interactions of \hierts with task $s$ until round $t$, and $H_t = (H_{s, t})_{s \in [m]}$ be the concatenation of all histories up to round $t$. For each task $s \in \cS_t$ in round $t$, \hierts samples $\theta_{s, t} \sim \prob{\theta_{s, *} = \cdot \mid H_t}$ and then takes action $A_{s, t} = \argmax_{a \in \cA} r(a; \theta_{s, t})$. The key difference from classical Thompson sampling is that the history $H_t$ includes observations of multiple tasks.

To sample $\theta_{s, t}$, we must address how the uncertainty over the unknown hyper-parameter $\mu_*$ and task parameters $\theta_{s, *}$ is modeled. The key idea is to maintain a \emph{hyper-posterior} $Q_t$ over $\mu_*$, given by
\begin{align*}
  Q_t(\mu)
  = \prob{\mu_* = \mu \mid H_t}\,,
\end{align*}
and then perform two-stage sampling. In particular, in round $t$, we first sample hyper-parameter $\mu_t \sim Q_t$. Next, for any task $s \in \cS_t$, we sample the task parameter $\theta_{s, t} \sim P_{s, t}(\cdot \mid \mu_t)$, where
\begin{align*}
  P_{s, t}(\theta \mid \mu)
  = \prob{\theta_{s, *} = \theta \mid \mu_* = \mu, H_{s, t}}\,.
\end{align*}
In $P_{s, t}(\cdot \mid \mu)$, we only condition on the history of task $s$, since $\theta_{s, *}$ is independent of the other task histories given $\mu_* = \mu$ (\cref{fig:setting}). This process clearly samples from the true posterior, which is given by
\begin{align}
  \condprob{\theta_{s, *} = \theta}{H_t}
  & = \int_\mu \condprob{\theta_{s, *} = \theta, \mu_* = \mu}{H_t} \dif \mu
  \label{eq:task posterior} \\
  & = \int_\mu P_{s, t}(\theta \mid \mu)
  Q_t(\mu) \dif \mu\,,
  \nonumber
\end{align}
where $P_{s, t}(\theta \mid \mu) \propto \cL_{s, t}(\theta) P(\theta \mid \mu)$ and
\begin{align*}
  \textstyle
  \cL_{s, t}(\theta)
  = \prod_{(a, y) \in H_{s, t}} P(y \mid a; \theta)
\end{align*}
denotes the likelihood of rewards in task $s$ given task parameter $\theta$.

The pseudo-code of \hierts is shown in \cref{alg:ts}. Sampling in \eqref{eq:task posterior} can be implemented exactly in Gaussian graphical models (\cref{sec:models}). These models have interpretable closed-form posteriors, which permit the regret analysis of \hierts. In practice, \hierts can be implemented for any posterior distributions, but may require approximate inference \citep{doucet01sequential} to tractably sample from the posterior.

\section{HIERARCHICAL GAUSSIAN BANDITS}
\label{sec:models}

Now we instantiate \hierts in hierarchical Gaussian models. This yields closed-form posteriors, which permit regret analysis (\cref{sec:key ideas}). We discuss generalization to other distributions in \cref{sec:extensions}.

We assume that the environment is generated as
\begin{align}
  \mu_*
  & \sim \cN(\mu_q, \Sigma_q)\,,
  \label{eq:gaussian hierarchical} \\
  \theta_{s, *} \mid \mu_*
  & \sim \cN(\mu_*, \Sigma_0)\,,
  & \forall s \in [m] \,,
  \nonumber \\
  Y_{s, t} \mid A_{s, t}, \theta_{s, *}
  & \sim \cN(A_{s, t}\T \theta_{s, *}, \sigma^2)\,,
  & \forall t \geq 1, \, s \in \cS_t\,,
  \nonumber
\end{align}
where $\Sigma_q \in \realset^{d \times d}$ and $\Sigma_0 \in \realset^{d \times d}$ are covariance matrices; $\mu_q$, $\mu_*$, $\theta_{s, *}$ are $d$-dimensional vectors; the set of actions is $\cA \subseteq \realset^d$; and the mean reward of action $a \in \cA$ is $r(a; \theta) = a\T \theta$. The reward noise is $\cN(0, \sigma^2)$. This formulation captures both the multi-armed and linear bandits, since the actions in the former can be viewed as vectors in a standard Euclidean basis. We assume that all of $\mu_q$, $\Sigma_q$, $\Sigma_0$, and $\sigma$ are known by the agent. This assumption is only needed in the analysis of \hierts, where we require an analytically tractable posterior. We relax it in our experiments (\cref{sec:experiments}), where we learn these quantities from past data.

\subsection{Gaussian Bandit}
\label{sec:gaussian bandit}

We start with a $K$-armed Gaussian bandit, which we instantiate as \eqref{eq:gaussian hierarchical} as follows. The task parameter $\theta_{s, *}$ is a vector of mean rewards in task $s$, where $\theta_{s, *, i}$ is the mean reward of action $i$. The covariance matrices are diagonal, $\Sigma_q = \sigma_q^2 I_K$ and $\Sigma_0 = \sigma_0^2 I_K$. We assume that both $\sigma_q > 0$ and $\sigma_0 > 0$ are known. The reward distribution of action $i$ is $\cN(\theta_{s, *, i}, \sigma^2)$, where $\sigma > 0$ is a known reward noise.

Because $\Sigma_q$ and $\Sigma_0$ are diagonal, the hyper-posterior in round $t$ factors across the actions. Specifically, it is $Q_t = \cN(\bar{\mu}_t, \bar{\Sigma}_t)$, where $\bar{\Sigma}_t = \diag{(\bar{\sigma}_{t, i}^2)_{i \in [K]}}$ and
\begin{align}
  \bar{\mu}_{t, i}
  & = \bar{\sigma}_{t, i}^2 \left(\frac{\mu_{q, i}}{\sigma_q^2} +
  \sum_{s \in [m]} \frac{N_{s, t, i}}{N_{s, t, i} \sigma_0^2 + \sigma^2}
  \frac{B_{s, t, i}}{N_{s, t, i}}\right)\,,
  \label{eq:mab hyperposterior} \\
  \bar{\sigma}_{t, i}^{-2}
  & = \sigma_q^{-2} + \sum_{s \in [m]} \frac{N_{s, t, i}}{N_{s, t, i} \sigma_0^2 + \sigma^2}\,.
  \nonumber
\end{align} 
Here $N_{s, t, i} = \sum_{\ell < t} \I{s \in \cS_\ell, A_{s, \ell} = i}$ is the number of times that action $i$ is taken in task $s$ up to round $t$ and $B_{s, t, i} = \sum_{\ell < t} \I{s \in \cS_\ell,  A_{s, \ell} = i} Y_{s, \ell}$ is its total reward. The hyper-posterior is derived in Appendix D of \citet{kveton21metathompson}. To understand it, it is helpful to view it as a Gaussian posterior where each task is a single observation. The observation of task $s$ is the empirical mean reward estimate of action $i$ in task $s$, $B_{s, t, i} / N_{s, t, i}$, and its variance is $(N_{s, t, i} \sigma_0^2 + \sigma^2) / N_{s, t, i}$. The tasks with more observations affect the value of $\bar{\mu}_{t, i}$ more, because their mean reward estimates have lower variances. The variance never decreases below $\sigma_0^2$, because even the actual mean reward $\theta_{s, *, i}$ would be a noisy observation of $\mu_{*, i}$ with variance $\sigma_0^2$.

After the hyper-parameter is sampled, $\mu_t \sim Q_t$, the task parameter is sampled, $\theta_{s, t} \sim \cN(\tilde{\mu}_{s, t}, \tilde{\Sigma}_{s, t})$, where $\tilde{\Sigma}_{s, t} = \diag{(\tilde{\sigma}_{s, t, i}^2)_{i \in [K]}}$ and
\begin{align}
  \tilde{\mu}_{s, t, i} 
  & = \tilde{\sigma}_{s, t, i}^2 \left(\frac{\mu_t}{\sigma_0^2} +
  \frac{B_{s, t, i}}{\sigma^2}\right)\,,
  \label{eq:mab conditional} \\
  \tilde{\sigma}_{s, t, i}^{-2}
  & = \frac{1}{\sigma_0^2} + \frac{N_{s, t, i}}{\sigma^2}\,.
  \nonumber
\end{align}
Note that the above is a Gaussian posterior with prior $\cN(\mu_t, \sigma_0^2 I_K)$ and $N_{s, t, i}$ observations.

\subsection{Linear Bandit with Gaussian Rewards}
\label{sec:linear bandit}

Now we study a $d$-dimensional linear bandit, which is instantiated as \eqref{eq:gaussian hierarchical} as follows. The task parameter $\theta_{s, *}$ are coefficients in a linear model. The covariance matrices $\Sigma_q$ are $\Sigma_0$ are positive semi-definite and known. The reward distribution of action $a$ is $\cN(a\T \theta_{s, *}, \sigma^2)$, where $\sigma > 0$ is a known reward noise.

Similarly to \cref{sec:gaussian bandit}, we obtain closed-form posteriors using \citet{kveton21metathompson}. The hyper-posterior in round $t$ is $Q_t = \cN(\bar{\mu}_t, \bar{\Sigma}_t)$, where
\begin{align}
  \bar{\mu}_t
  & = \bar{\Sigma}_t \Big(\Sigma_q^{-1} \mu_q +
  \smashoperator{\sum_{s \in [m]}}
  B_{s, t} - G_{s, t}(\Sigma_0^{-1} + G_{s, t})^{-1} B_{s, t}\Big)
  \nonumber \\
  & = \bar{\Sigma}_t \Big(\Sigma_q^{-1} \mu_q +
  \sum_{s \in [m]} (\Sigma_0 + G_{s, t}^{-1})^{-1} G_{s, t}^{-1} B_{s, t}\Big)\,,
  \nonumber \\
  \bar{\Sigma}_t^{-1}
  & = \Sigma_q^{-1} +
  \sum_{s \in [m]} G_{s, t} - G_{s, t}(\Sigma_0^{-1} + G_{s, t})^{-1} G_{s, t}
  \nonumber \\
  & = \Sigma_q^{-1} +
  \sum_{s \in [m]} (\Sigma_0 + G_{s, t}^{-1})^{-1}\,.
  \label{eq:linear hyperposterior}
\end{align}
Here
\begin{align*}
  G_{s, t}
  = \sigma^{-2} \sum_{\ell < t} \I{s \in \cS_\ell} A_{s, \ell} A_{s, \ell}\T
\end{align*}
is the outer product of the features of taken actions in task $s$ up to round $t$ and
\begin{align*}
  B_{s, t}
  = \sigma^{-2} \sum_{\ell < t} \I{s \in \cS_\ell} A_{s, \ell} Y_{s, \ell}
\end{align*}
is their sum weighted by the observed rewards. Similarly to \eqref{eq:mab hyperposterior}, it is helpful to view \eqref{eq:linear hyperposterior} as a multivariate Gaussian posterior where each task is a single observation. The observation of task $s$ is the least squares estimate of $\theta_{s, *}$ from task $s$, $G_{s, t}^{-1} B_{s, t}$, and its covariance is $\Sigma_0 + G_{s, t}^{-1}$. Again, the tasks with many observations affect the value of $\bar{\mu}_t$ more, because $G_{s, t}^{-1}$ approaches a zero matrix in these tasks. In this setting, the covariance approaches $\Sigma_0$, because even the unknown task parameter $\theta_{s, *}$ would be a noisy observation of $\mu_*$ with covariance $\Sigma_0$.

After the hyper-parameter is sampled, $\mu_t \sim Q_t$, the task parameter is sampled, $\theta_{s, t} \sim \cN(\tilde{\mu}_{s, t}, \tilde{\Sigma}_{s, t})$, where
\begin{align}
  \tilde{\mu}_{s, t} 
  & = \tilde{\Sigma}_{s, t} \left(\Sigma_0^{-1} \mu_t + B_{s, t}\right)\,,
  \label{eq:linear conditional} \\
  \tilde{\Sigma}_{s, t}^{-1}
  & = \Sigma_0^{-1} + G_{s, t}\,.
  \nonumber
\end{align}
The above is the posterior of a linear model with a Gaussian prior $\cN(\mu_t, \Sigma_0)$ and Gaussian observations.

\section{KEY IDEAS IN OUR ANALYSES}
\label{sec:key ideas}

This section reviews key ideas in our regret analyses, including a novel variance decomposition for the posterior of a hierarchical Gaussian model. Due to space constraints, we only discuss the linear bandit in \cref{sec:linear bandit}.

\subsection{Bayes Regret Bound}

Fix round $t$ and task $s \in \cS_t$. Since \hierts is a posterior sampling algorithm, both the posterior sample $\theta_{s, t}$ and the unknown task parameter $\theta_{s, *}$ are i.i.d.\ conditioned on $H_t$. Moreover, \eqref{eq:task posterior} is a marginalization and conditioning in a hierarchical Gaussian model given in \cref{fig:setting}. Therefore, although we never explicitly derive $\theta_{s, *} \mid H_t$, we know that it is a multivariate Gaussian distribution \citep{koller09probabilistic}; and we denote it by $\condprob{\theta_{s, *} = \theta}{H_t} = \cN(\theta; \hat{\mu}_{s, t}, \hat{\Sigma}_{s, t})$.

Following existing Bayes regret analyses \citep{russo14learning}, we have that
\begin{align*}
  & \E{}{A_{s, *}\T \theta_{s, *} - A_{s, t}\T \theta_{s, *} \mid H_t} = \\
  & \E{}{A_{s, *}\T (\theta_{s, *} - \hat{\mu}_{s, t}) \mid H_t} +
  \E{}{A_{s, t}\T (\hat{\mu}_{s, t}- \theta_{s, *}) \mid H_t}\,.
\end{align*}
Conditioned on history $H_t$, we observe that $\hat{\mu}_{s, t} - \theta_{s, *}$ is a zero-mean random vector and that $A_{s, t}$ is independent of it. Hence $\condE{A_{s, t}\T (\hat{\mu}_{s, t} - \theta_{s, *})}{H_t} = 0$ and the Bayes regret is bounded as
\begin{align*}
  \Bregret(m, n)
  \leq \E{}{\sum_{t \geq 1}\sum_{s \in \cS_t}
  \E{}{A_{s, *}\T (\theta_{s, *} - \hat{\mu}_{s, t}) \mid H_t}}\,.
\end{align*}
The following lemma provides an upper bound on the Bayes regret for $m$ tasks, with at most $n$ interactions with each, using the sum of posterior variances
\begin{align}
  \cV(m, n)
  = \E{}{\sum_{t \geq 1} \sum_{s \in \cS_t} \normw{A_{s, t}}{\hat{\Sigma}_{s, t}}^2}\,.
  \label{eq:posterior variances}
\end{align}
The proof is deferred to \cref{sec:bayes regret proof}.

\begin{lemma}
\label{lem:bayes regret} For any $\delta > 0$, the Bayes regret $\Bregret(m, n)$ in a hierarchical linear bandit (\cref{sec:linear bandit}) is bounded by
\begin{align*}
  \sqrt{2 d m n \cV(m, n) \log(1 / \delta)} +
  \sqrt{2 / \pi} \sigma_{\max} d^\frac{3}{2} m n \delta\,,
\end{align*}
where $\sigma_{\max}^2 = \lambda_1(\Sigma_0) + \lambda_1^2(\Sigma_0) \lambda_1(\Sigma_q) / \lambda_d^2(\Sigma_0)$. When the action space is finite, $|\cA| = K$, we also get
\begin{align*}
  \sqrt{2 m n \cV(m, n) \log(1 / \delta)} +
  \sqrt{2 / \pi} \sigma_{\max} K m n \delta\,.
\end{align*}
\end{lemma}

Therefore, to bound the regret, we only need to bound the posterior variances induced by the taken actions. The main challenge is that our posterior is over multiple variables. As can be seen in \eqref{eq:task posterior}, it comprises the hyper-posterior $Q_t$ over $\mu_*$ and the conditional $P_{s, t}$ over $\theta_{s, *}$. For any fixed $\mu_*$, $P_{s, t}(\cdot \mid \mu_*)$ should concentrate at $\theta_{s, *}$ as the agent gets more observations from task $s$. In addition, $Q_t$ should concentrate at $\mu_*$ as the agent learns more about $\mu_*$ from all tasks.

\subsection{Total Variance Decomposition}
\label{sec:total variance decomposition}

Due to the hierarchical structure of our problem, it is difficult to reason about the rate at which $\hat{\Sigma}_{s, t}$ \say{decreases}. In this work, we propose a novel variance decomposition that allows this. The decomposition uses the law of total variance \citep{weiss05probability}, which states that for any $X$ and $Y$,
\begin{align*}
  \var{X}
  = \E{}{\condvar{X}{Y}} + \var{\condE{X}{Y}}\,.
\end{align*}
If $X = \theta$ was a scalar task parameter and $Y = \mu$ was a scalar hyper-parameter, and we conditioned on $H$, the law would give
\begin{align*}
  \condvar{\theta}{H}
  = \condE{\condvar{\theta}{\mu, H}}{H} +
  \condvar{\condE{\theta}{\mu, H}}{H}\,.
\end{align*}
This law extends to covariances \citep{weiss05probability}, where the conditional variance $\condvar{\cdot}{H}$ is substituted with the covariance $\condcov{\cdot}{H}$. We show the decomposition for a hierarchical Gaussian model below.

\subsection{Hierarchical Gaussian Models}

Recall that $\hat{\Sigma}_{s, t} = \condcov{\theta_{s, *}}{H_t}$. We derive a general formula for decomposing $\condcov{\theta_{s, *}}{H_t}$ below. To simplify notation, we consider a fixed task $s$ and round $t$, and drop subindexing by them.

\begin{lemma}
\label{lem:covariance decomposition} Let $\theta \mid \mu \sim \cN(\mu, \Sigma_0)$ and $H = (x_t, Y_t)_{t = 1}^n$ be $n$ observations generated as $Y_t \mid \theta, x_t \sim \cN(x_t\T \theta, \sigma^2)$. Let $\condprob{\mu}{H} = \cN(\mu; \bar{\mu}, \bar{\Sigma})$. Then
\begin{align*}
  \condcov{\theta}{H}
  = {} & (\Sigma_0^{-1} + G)^{-1} + {} \\
  & (\Sigma_0^{-1} + G)^{-1} \Sigma_0^{-1} \bar{\Sigma}
  \Sigma_0^{-1} (\Sigma_0^{-1} + G)^{-1}\,,
\end{align*}
where $G = \sigma^{-2} \sum_{t = 1}^n x_t x_t\T$. Moreover, for any $x \in \realset^d$,
\begin{align*}
  & x\T (\Sigma_0^{-1} + G)^{-1} \Sigma_0^{-1} \bar{\Sigma}
  \Sigma_0^{-1} (\Sigma_0^{-1} + G)^{-1} x \\
  & \quad \leq \frac{\lambda_1^2(\Sigma_0) \lambda_1(\bar{\Sigma})}{\lambda_d^2(\Sigma_0)}
  \normw{x}{2}^2\,.
\end{align*}
\end{lemma}
\begin{proof}
By definition,
\begin{align*}
  \condcov{\theta}{\mu, H}
  & = (\Sigma_0^{-1} + G)^{-1}\,, \\
  \condE{\theta}{\mu, H}
  & = \condcov{\theta}{\mu, H}
  (\Sigma_0^{-1} \mu + B)\,,
\end{align*}
where $B = \sigma^{-2}\sum_{t = 1}^n x_t Y_t$. Because $\condcov{\theta}{\mu, H}$ does not depend on $\mu$, $\condE{\condcov{\theta}{\mu, H}}{H} = \condcov{\theta}{\mu, H}$. In addition, since $B$ is a constant conditioned on $H$,
\begin{align*}
  & \condcov{\condE{\theta}{\mu, H}}{H} \\
  & \quad = \condcov{\condcov{\theta}{\mu, H} \Sigma_0^{-1} \mu}{H} \\
  & \quad = (\Sigma_0^{-1} + G)^{-1} \Sigma_0^{-1} \bar{\Sigma}
  \Sigma_0^{-1} (\Sigma_0^{-1} + G)^{-1}\,.
\end{align*}
This proves the first claim. The second claim follows from standard norm and eigenvalue inequalities.
\end{proof}

We use \cref{lem:covariance decomposition} as follows. For task $s$ and round $t$, the posterior covariance decomposes as
\begin{align}
  \hat{\Sigma}_{s, t}
  = {} & (\Sigma_0^{-1} + G_{s, t})^{-1} + {}
  \label{eq:covariance decomposition} \\
  & (\Sigma_0^{-1} + G_{s, t})^{-1} \Sigma_0^{-1} \bar{\Sigma}_t
  \Sigma_0^{-1} (\Sigma_0^{-1} + G_{s, t})^{-1}\,.
  \nonumber
\end{align}
The first term is $\condcov{\theta_{s, *}}{\mu_*, H_t}$ and captures uncertainty in $\theta_{s, *}$ conditioned on $\mu_*$. The second term depends on hyper-posterior covariance $\bar{\Sigma}_t$ and represents uncertainty in $\mu_*$. Since the first term is exactly $\tilde{\Sigma}_{s, t}$ in \eqref{eq:linear conditional}, while the second term is weighted by it, both are small when we get enough observations for task $s$. The above also says that $\normw{A_{s, t}}{\hat{\Sigma}_{s, t}}^2 = A_{s, t}\T \hat{\Sigma}_{s, t} A_{s, t}$ in \eqref{eq:posterior variances} decompose into the two respective norms, which yields our regret decomposition.

\subsection{Extensions}
\label{sec:extensions}

So far, we only focused on hierarchical Gaussian models with known hyper-prior and task prior covariances. This is only because they have closed-form posteriors that are easy to interpret and manipulate, without resorting to approximations \citep{doucet01sequential}. This choice simplifies algebra and allows us to focus on the key hierarchical structure of our problem. We believe that the tools developed in this section can be applied more broadly. We discuss this next.

\cref{lem:bayes regret} decomposes the Bayes regret into posterior variances and upper bounds on the regret due to tail events. The posterior variance can be derived for any exponential-family posterior with a conjugate prior. On the other hand, the tail inequalities require sub-Gaussianity, which is a property of many exponential-family distributions.

\cref{lem:covariance decomposition} decomposes the posterior covariance in a hierarchical Gaussian model. It relies on the law of total covariance, which holds for any distribution, to obtain the task and hyper-parameter uncertainties. We expect that similar lemmas can be proved for other hierarchical models, so long as closed-form expressions for the respective uncertainties exist. Another notable property of our decomposition is that it does not require the marginal posterior of $\theta_{s, *}$. We view it as a strength. It means that our approach can be applied to complex graphical models where the marginal uncertainty may be hard to express, but the conditional and prior uncertainties are readily available.

One limitation of our analyses is that we bound the Bayes regret, instead of a stronger frequentist regret. This simplifies our proofs while they still capture our problem structure. Our analyses can be extended to the frequentist setting. This only requires a new proof of \cref{lem:bayes regret}, with martingale bounds for tail events and anti-concentration bounds for posterior sampling. The rest of the analysis, where our main contributions are, would not change.

\section{REGRET BOUNDS}
\label{sec:regret bounds}

This section bounds the Bayes regret of \hierts in the linear bandit in \cref{sec:linear bandit}. Our bounds are specialized to multi-armed bandits in \cref{sec:mab bounds}. The key idea is to bound the posterior variances $\cV(m, n)$ in \eqref{eq:posterior variances} and then substitute the bound into the infinite-action bound in \cref{lem:bayes regret}. We bound the variances using the total covariance decomposition in \cref{sec:total variance decomposition}. Without loss of generality, we assume that the action set $\cA$ is a subset of a unit ball, that is $\max_{a \in \cA} \normw{a}{2} \leq 1$ for any action $a \in \cA$.

We make the following contributions in theory. First, we prove regret bounds using a novel variance decomposition (\cref{sec:total variance decomposition}), which improves in constant factors over classical information-theory bounds \citep{russo16information}. Second, we prove the first Bayes regret bound for the setting where an agent that interacts with multiple tasks simultaneously.

This section has two parts. In \cref{sec:sequential regret}, we assume that only one action is taken in any round $t$, $|\cS_t| = 1$. We call this setting \emph{sequential}, and note that it is the primary setting studied by prior works \citep{kveton21metathompson,basu21noregrets}. In \cref{sec:concurrent regret}, we focus on a \emph{concurrent} setting, where a single action can be taken in up to $L$ tasks in any round $t$, $|\cS_t| \leq L \leq m$. The challenge of this setting is that the task parameters are only updated after all actions are taken.

\subsection{Sequential Regret}
\label{sec:sequential regret}

The following theorem provides a regret bound for the sequential setting.

\begin{theorem}[Sequential regret]
\label{thm:sequential regret} Let $|\cS_t| = 1$ for all rounds $t$ and $\cA \subseteq \realset^d$. Choose $\delta = 1 / (m n)$. Then the Bayes regret of \hierts is
\begin{align*}
  \Bregret(m, n)
  \leq d \sqrt{2 m n [c_1 m + c_2] \log(m n)} + c_3\,,
\end{align*}
where $c_3 = O(d^\frac{3}{2})$,
\begin{align*}
  c_1
  & = \frac{\lambda_1(\Sigma_0)}{\log(1 + \sigma^{-2} \lambda_1(\Sigma_0))}
  \log\left(1 + \frac{\lambda_1(\Sigma_0) n}{\sigma^2 d}\right)\,, \\
  c_2
  & = \frac{c_q c}{\log(1 + \sigma^{-2} c_q)}
  \log\left(1 + \frac{\lambda_1(\Sigma_q) m}{\lambda_d(\Sigma_0)}\right)\,, \\
  c_q
  & = \frac{\lambda_1^2(\Sigma_0) \lambda_1(\Sigma_q)}{\lambda_d^2(\Sigma_0)}, \quad 
  c = 1 + \sigma^{-2} \lambda_1(\Sigma_0)\,.
\end{align*}
\end{theorem}

The proof of \cref{thm:sequential regret} is based on three steps. First, we use \cref{lem:bayes regret}. Second, we employ \cref{lem:covariance decomposition} to decompose the posterior variance in any round into that of the task parameters and hyper-parameter. Finally, we apply elliptical lemmas to bound each term separately. \cref{thm:sequential regret} has a nice interpretation: $m \sqrt{c_1 n}$ is the regret for learning task parameters and $\sqrt{c_2 m n}$ is the regret for learning the hyper-parameter $\mu_*$. We elaborate on both terms below.

The term $m \sqrt{c_1 n}$ represents the regret for solving $m$ bandit tasks, which are sampled i.i.d.\ from a known prior $\cN(\mu_*, \Sigma_0)$. Under this assumption, no task provides information about any other task, and thus the term is linear in $m$. The constant $c_1$ is $O(\lambda_1(\Sigma_0))$ and reflects the dependence on the prior width $\sqrt{\lambda_1(\Sigma_0)}$. Roughly speaking, when the task prior is half as informative, $\sqrt{c_1}$ doubles and so does $m \sqrt{c_1 n}$. This is the expected scaling with conditional uncertainty of $\theta_{s, *}$ given $\mu_*$.

The term $\sqrt{c_2 m n}$ is the regret for learning the hyper-parameter. Asymptotically, it is $O(\sqrt{m})$ smaller than $m \sqrt{c_1 n}$. Therefore, for a large number of tasks $m$, its contribution to the total regret is negligible. This is why hierarchical Bayesian bandits perform so well in practice. The constant $c_2$ is $O(\lambda_1(\Sigma_q))$ and reflects the dependence on the hyper-prior width $\sqrt{\lambda_1(\Sigma_0)}$. When the hyper-prior is half as informative, $\sqrt{c_2}$ doubles and so does $\sqrt{c_2 m n}$. This is the expected scaling with the marginal uncertainty of $\mu_*$.

\subsection{Tightness of Regret Bounds}
\label{sec:tightness}

One shortcoming of our current analysis is that we do not provide a matching lower bound.
To the best of our knowledge, Bayes regret lower bounds are rare and do not match existing upper bounds. The only lower bound that we are aware of is $\Omega(\log^2 n)$ in Theorem 3 of \citet{lai87adaptive}. The bound is for $K$-armed bandits and it is unclear how to apply it to structured problems. Seminal works on Bayes regret minimization \citep{russo14learning,russo16information} do not match it. Therefore, to show that our problem structure is reflected in our bound, we compare the regret of \hierts to baselines that have more information or use less structure.

Now we compare the regret of \hierts to two \lints \citep{agrawal13thompson} baselines that do use our hierarchical model. This first is an oracle \lints that knows $\mu_*$, and so has more information than \hierts. Its Bayes regret would be as in \cref{thm:sequential regret} with $c_2 = 0$. Not surprisingly, it is lower than that of \hierts. The second baseline is \lints that knows that $\mu_* \sim \cN(\mu_q, \Sigma_q)$, but does not model the structure that the tasks share $\mu_*$. In this case, each task parameter can be viewed as having prior $\cN(\mu_q, \Sigma_q + \Sigma_0)$. The regret of this algorithm would be as in \cref{thm:sequential regret} with $c_2 = 0$, while $\lambda_1(\Sigma_0)$ in $c_1$ would be $\lambda_1(\Sigma_q + \Sigma_0)$. Since $c_1$ is multiplied by $m$ while $c_2$ is not, \hierts would have lower regret as $m \to \infty$. This is a powerful testament to the benefit of learning the hyper-parameter.

Finally, we want to comment on linear dependence in $d$ and $m$ in \cref{thm:sequential regret}. The dependence on $d$ is standard in Bayes regret analyses for linear bandits with infinitely many arms \citep{russo14learning,lu19informationtheoretic}. As for the number of tasks $m$, since the tasks are drawn i.i.d.\ from the same hyper-prior, they do not provide any additional information about each other. So, even if the hyper-parameter is known, the regret for learning to act in $m$ tasks with $n$ rounds would be $O(m \sqrt{n})$. Our improvements are in constants due to better variance attribution. Other bandit meta-learning works \citep{kveton21metathompson,basu21noregrets} made similar observations. Also note that the frequentist regret of \lints applied to $m$ independent linear bandit tasks is $\tilde{O}(m d^\frac{3}{2} \sqrt{n})$ \citep{agrawal13thompson}. This is worse by a factor of $\sqrt{d}$ than the bound in \cref{thm:sequential regret}.

\subsection{Concurrent Regret}
\label{sec:concurrent regret}

Now we investigate the concurrent setting, where the agent acts in up to $L$ tasks per round. This setting is challenging because the hyper-posterior $Q_t$ is not updated until the end of the round. This is because the task posteriors are not refined with the observations from concurrent tasks. This delayed feedback should increase regret. Before we show it, we make the following assumption on the action space.

\begin{assumption}
\label{ass:basis} There exist actions $\{a_i\}_{i = 1}^d \subseteq \cA$ and $\eta > 0$ such that $\lambda_d(\sum_{i = 1}^d a_i a_i\T) \geq \eta$.
\end{assumption}

This assumption is without loss of generality. Specifically, if $\realset^d$ was not spanned by actions in $\cA$, we could project $\cA$ into a subspace where the assumption would hold. Our regret bound is below.

\begin{theorem}[Concurrent regret]
\label{thm:concurrent regret} Let $|\cS_t| \leq L \leq m$ and $\cA \subseteq \realset^d$. Let $\delta = 1 / (m n)$. Then the Bayes regret of \hierts is
\begin{align*}
  \Bregret(m, n)
  \leq d \sqrt{2 m n [c_1 m + c_2] \log(m n)} + c_3\,,
\end{align*}
where $c_1$, $c_q$, and $c$ are defined as in \cref{thm:sequential regret},
\begin{align*}
  c_2
  & = \frac{c_q c_4 c}{\log(1 + \sigma^{-2} c_q)}
  \log\left(1 + \frac{\lambda_1(\Sigma_q) m}{\lambda_d(\Sigma_0)}\right)\,, \\
  c_4
  & = 1 + \frac{\sigma^{-2} \lambda_1(\Sigma_q) (\lambda_1(\Sigma_0) + \sigma^2 / \eta)}
  {\lambda_1(\Sigma_q) + (\lambda_1(\Sigma_0) + \sigma^2 / \eta) / L}\,,
\end{align*}
and $c_3 = O(d^\frac{3}{2} m)$.
\end{theorem}

The key step in the proof is to modify \hierts as follows. For the first $d$ interactions with any task $s$, we take actions $\{a_i\}_{i = 1}^d$. This guarantees that we explore all directions within the task, and allows us to bound losses from not updating the task posterior with concurrent observations. This modification of \hierts is trivial and analogous to popular initialization in bandits, where each arm is pulled once in the first rounds \citep{auer02finitetime}.

The regret bound in \cref{thm:concurrent regret} is similar to that in \cref{thm:sequential regret}. There are two key differences. First, the additional scaling factor $c_4$ in $c_2$ is the price for taking concurrent actions. It increases as more actions $L$ are taken concurrently, but is sublinear in $L$. Second, $c_3$ arises due to trivially bounding $dm$ rounds of forced exploration. To the best of our knowledge, \cref{thm:concurrent regret} is the first Bayes regret bound where multiple bandit tasks are solved concurrently. Prior works only proved frequentist regret bounds \citep{yang21impact}.

\section{EXPERIMENTS}
\label{sec:experiments}

\begin{figure*}[t!]
  \centering
  \begin{minipage}{0.32\textwidth}
    \includegraphics[width=\linewidth]{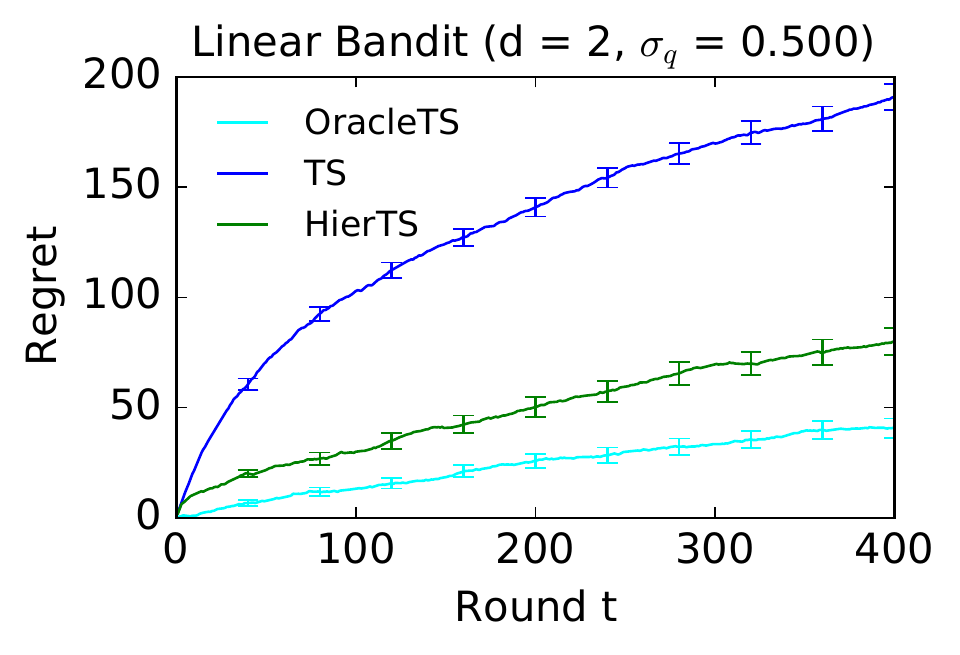}
  \end{minipage}
  \begin{minipage}{0.32\textwidth}
    \includegraphics[width=\linewidth]{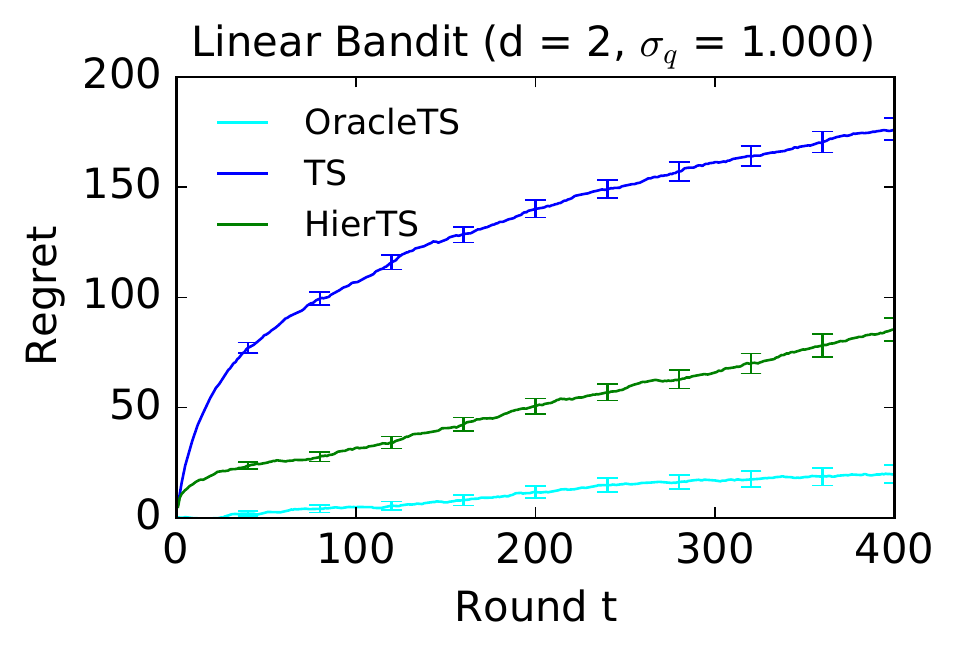}
  \end{minipage}
  \begin{minipage}{0.32\textwidth}
    \includegraphics[width=\linewidth]{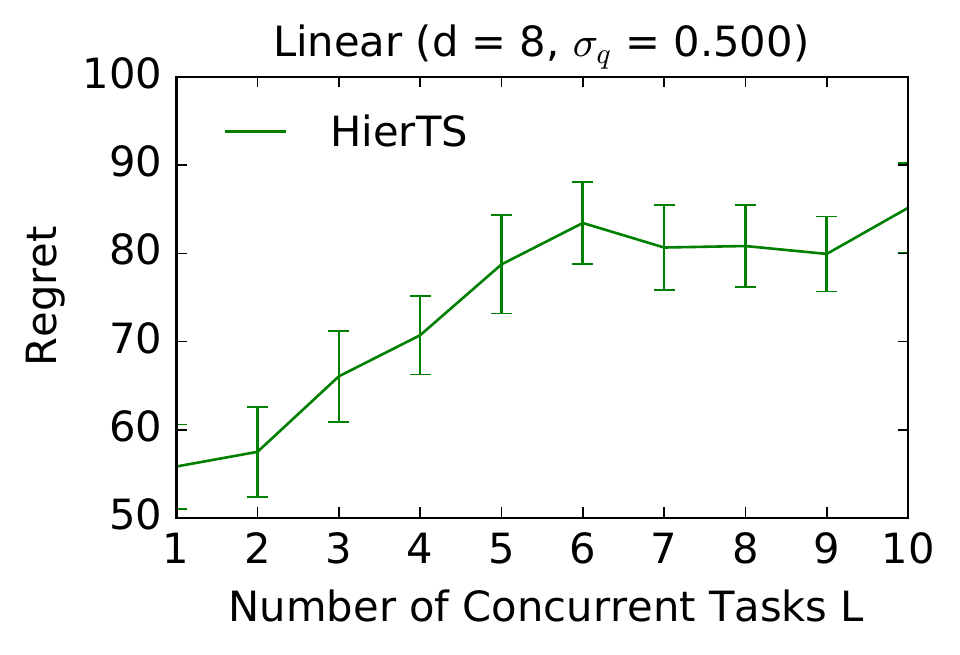}
  \end{minipage}
  \vspace{-0.1in}
  \caption{Evaluation of \hierts on synthetic bandit problems. From left to right, we report the Bayes regret (a) for smaller $\sigma_q$, (b) for larger $\sigma_q$, (c) and as a function of the number of concurrent tasks $L$.}
  \label{fig:synthetic}
\end{figure*}

We compare \hierts to two TS baselines (\cref{sec:tightness}) that do not learn the hyper-parameter $\mu_*$. The first baseline is an idealized algorithm that knows $\mu_*$ and uses the true prior $\cN(\mu_*, \Sigma_0)$. We call it \oraclets. As \oraclets has more information than \hierts, we expect it to outperform \hierts. The second baseline, which we call \ts, ignores that $\mu_*$ is shared among the tasks and uses the marginal prior of $\theta_{s, *}$, $\cN(\mu_q, \Sigma_q + \Sigma_0)$, in each task. 

We experiment with two linear bandit problems with $m = 10$ tasks: a synthetic problem with Gaussian rewards and an online image classification problem. The former is used to validate our regret bounds. The latter has non-Gaussian rewards and demonstrates that \hierts is robust to prior misspecification. Our setup closely follows \citet{basu21noregrets}. However, our tasks can arrive in an arbitrary order and in parallel. Due to space constraints, we only report the synthetic experiment here, and defer the rest to \cref{sec:classification experiments}.

The synthethic problem is defined as follows: $d = 2$, $\abs{\cA} = 10$, and each action is sampled uniformly from $[-0.5, 0.5]^d$. Initially, the number of concurrent tasks is $L = 5$; but we vary it later to measure its impact on regret. The number of rounds is $n = 200 m / L$ and $\cS_t$ is defined as follows. First, we take a random permutation of the list of tasks where each task appears exactly $200$ times. Then we batch every $L$ consecutive elements of the list and set $\cS_t$ to the $t$-th batch. The hyper-prior is $\cN(\mathbf{0}, \Sigma_q)$ with $\Sigma_q = \sigma_q^2 I_d$, the task covariance is $\Sigma_0 = \sigma_0^2 I_d$, and the reward noise is $\sigma = 0.5$. We choose $\sigma_q \in \set{0.5, 1}$ and $\sigma_0 = 0.1$, where $\sigma_q \gg \sigma_0$ so that the effect of learning $\mu_*$ on faster learning of $\theta_{s, *}$ is easier to measure.

The regret of all compared algorithms is reported in \cref{fig:synthetic}. In plots (a) and (b), we show how the regret scales with the number of rounds for small ($\sigma_q = 0.5$) and large ($\sigma_q$ = 1) hyper-prior width. As suggested in \cref{sec:sequential regret}, \hierts outperforms \ts that does not try to learn $\mu_*$. It is comparable to \oraclets when $\sigma_q$ is small, but degrades as $\sigma_q$ increases. This matches the regret bound in \cref{thm:concurrent regret}, where $c_2$ grows with $\sigma_q$. In plot (c), we show how the regret of \hierts varies with the number of concurrent tasks $L$. We observe that it increases with $L$, but the increase is sublinear, as suggested in \cref{sec:concurrent regret}.

\section{RELATED WORK}
\label{sec:related work}

The most related works are recent papers on bandit meta-learning \citep{bastani19meta,ortega19metalearning,cella20metalearning,kveton21metathompson,basu21noregrets,peleg21metalearning,simchowitz21bayesian}, where a learning agent interacts with a single task at a time until completion. Both \citet{kveton21metathompson} and \citet{basu21noregrets} represent their problems using graphical models and apply Thompson sampling to solve them. The setting of these papers is less general than ours. \citet{wan21metadatabased} study a setting where the tasks can arrive in any order. We differ from this work in several aspects. First, they only consider a $K$-armed bandit. Second, their model is different. In our notation, \citet{wan21metadatabased} assume that the mean reward of action $a$ in task $s$ is $x_{s, a}\T \mu_*$ plus i.i.d.\ noise, where $x_{s, a}$ is an observed feature vector. The i.i.d.\ noise prevents generalization to a large number of actions. In our work, the mean reward of action $a$ in task $s$ is $a\T \theta_{s, *}$, where $\theta_{s, *} \sim \cN(\mu_*, \Sigma_0)$. Third, \citet{wan21metadatabased} derive a frequentist regret bound, which matches \cref{thm:sequential regret} asymptotically, but does not explicitly depend on prior widths. Finally, \citet{wan21metadatabased} do not consider the concurrent setting. To the best of our knowledge, we are the first to study Bayesian bandits with arbitrarily ordered and concurrent tasks.

The novelty in our analysis is the total covariance decomposition, which leads to better variance attribution in structured models than information-theoretic bounds \citep{russo16information,lu19informationtheoretic,basu21noregrets}. For instance, take Theorem 5 of \citet{basu21noregrets}, which corresponds to our sequential meta-learning setting. Forced exploration is needed to make their task term $O(\lambda_1(\Sigma_0))$. This is because the upper bound on the regret with filtered mutual information depends on the maximum marginal task parameter covariance, which can be $\lambda_1(\Sigma_q + \Sigma_0)$. In our analysis, the comparable term $c_1$ (\cref{thm:sequential regret}) is $O(\lambda_1(\Sigma_0))$ without any forced exploration. We also improve upon related analysis of \citet{kveton21metathompson} in several aspects. First, \citet{kveton21metathompson} analyze only a $K$-armed bandit. Second, they derive that the additional regret for meta-learning is $\tilde{O}(\sqrt{m} n^2)$; while our bound shows $\tilde{O}(\sqrt{m n})$. Finally, our setting generalizes bandit meta-learning.

Meta- and multi-task bandits have also been studied in the frequentist setting \citep{azar13sequential,deshmukh17multitask}. \citet{cella20metalearning} propose a \linucb algorithm \citep{abbasi-yadkori11improved} that constructs an ellipsoid around the unknown hyper-parameter in a linear bandit. The concurrent setting has also been studied, but with a different shared structure of task parameters. \citet{dubey20kernel} use a kernel matrix, \citet{wang21multitask} utilize pairwise distances of task parameters, and \citet{yang21impact} use low-rank factorization. Our structure, where the task parameters are drawn from an unknown prior, is both novel and important to study because it differs significantly from the aforementioned works. Earlier works on bandits with similar instances rely on clustering \citep{gentile14online,gentile17clustering,li16collaborative} and low-rank factorization \citep{kawale15efficient,sen17contextual,katariya16dcm,katariya17stochastic}. They analyze the frequentist regret, which is a stronger metric than the Bayes regret. Except for one work, all algorithms are UCB-like and conservative in practice. In comparison, \hierts uses a natural stochastic structure. This makes it practical, to the point that the analyzed algorithm performs well in practice without any additional tuning.

Another related line of work are latent bandits \citep{maillard14latent,hong20latent,hong22thompson}, where the bandit problem is parameterized by an unknown latent state. If known, the latent state could help the agent to identify the bandit instance that it interacts with. These works reason about latent variables; but the purpose is different from our work, where we introduce the unknown hyper-parameter $\mu_*$ to relate multiple similar tasks.

\section{CONCLUSIONS}
\label{sec:conclusions}

We study \emph{hierarchical Bayesian  bandits}, a general setting for solving similar bandit tasks. Instances of our setting recover meta-, multi-task, and federated bandits in prior works. We propose a natural hierarchical Thompson sampling algorithm, which can be implemented exactly and analyzed in Gaussian models. We analyze it using a novel total variance decomposition, which leads to interpretable regret bounds that scale with the hyper-prior and task prior widths. The benefit of hierarchical models is shown in both synthetic and real-world domains.

While we view our work as solving an extremely general problem, there are multiple directions for future work. For instance, we only study a specific hierarchical Gaussian structure in \cref{sec:models}. However, based on the discussion in \cref{sec:extensions}, we believe that our tools would apply to arbitrary graphical models with general sub-Gaussian distributions. Another direction for future work are frequentist upper bounds and matching lower bounds, in both the frequentist and Bayesian settings.

\bibliographystyle{abbrvnat}
\bibliography{references}

\clearpage
\onecolumn
\appendix

\section{Proof of \cref{lem:bayes regret}}
\label{sec:bayes regret proof}

The first claim is proved as follows. Fix round $t$ and task $s \in \cS_t$. Since $\hat{\mu}_{s, t}$ is a deterministic function of $H_t$, and $A_{s, *}$ and $A_{s, t}$ are i.i.d.\ given $H_t$, we have
\begin{align*}
  \E{}{A_{s, *}\T \theta_{s, *} - A_{s, t}\T \theta_{s, *}}
  = \E{}{\condE{A_{s, *}\T (\theta_{s, *} - \hat{\mu}_{s, t})}{H_t}} +
  \E{}{\condE{A_{s, t}\T (\hat{\mu}_{s, t} - \theta_{s, *})}{H_t}}\,.
\end{align*}
Moreover, $\theta_{s, *} - \hat{\mu}_{s, t}$ is a zero-mean random vector independent of $A_{s, t}$, and thus $\condE{A_{s, t}\T (\hat{\mu}_{s, t} - \theta_{s, *})}{H_t} = 0$. So we only need to bound the first term above. Let
\begin{align*}
  E_{s, t} =
  \set{\normw{\theta_{s, *} - \hat{\mu}_{s, t}}{\hat{\Sigma}_{s, t}^{-1}}
  \leq \sqrt{2 d \log(1 / \delta)}}
\end{align*}
be the event that a high-probability confidence interval for the task parameter $\theta_{s, *}$ holds. Fix history $H_t$. Then by the Cauchy-Schwarz inequality,
\begin{align*}
  \condE{A_{s, *}\T (\theta_{s, *} - \hat{\mu}_{s, t})}{H_t}
  & \leq \condE{\normw{A_{s, *}}{\hat{\Sigma}_{s, t}}
  \normw{\theta_{s, *} - \hat{\mu}_{s, t}}{\hat{\Sigma}_{s, t}^{-1}}}{H_t} \\
  & \leq \sqrt{2 d \log(1 / \delta)} \, \condE{\normw{A_{s, *}}{\hat{\Sigma}_{s, t}}}{H_t} +
  \underbrace{\max_{a \in \cA} \normw{a}{\hat{\Sigma}_{s, t}}}_{\leq \sigma_{\max}}
  \condE{\normw{\theta_{s, *} - \hat{\mu}_{s, t}}{\hat{\Sigma}_{s, t}^{-1}}
  \I{\bar{E}_{s, t}}}{H_t} \\
  & = \sqrt{2 d \log(1 / \delta)} \, \condE{\normw{A_{s, t}}{\hat{\Sigma}_{s, t}}}{H_t} +
  \sigma_{\max} \, \condE{\normw{\theta_{s, *} - \hat{\mu}_{s, t}}{\hat{\Sigma}_{s, t}^{-1}}
  \I{\bar{E}_{s, t}}}{H_t}\,.
\end{align*}
The equality follows from the fact that $\hat{\Sigma}_{s, t}$ is a deterministic function of $H_t$, and that $A_{s, *}$ and $A_{s, t}$ are i.i.d.\ given $H_t$. Now we focus on the second term above. First, note that
\begin{align*}
  \normw{\theta_{s, *} - \hat{\mu}_{s, t}}{\hat{\Sigma}_{s, t}^{-1}}
  = \normw{\hat{\Sigma}^{- \frac{1}{2}}_{s, t} (\theta_{s, *} - \hat{\mu}_{s, t})}{2}
  \leq \sqrt{d} \maxnorm{\hat{\Sigma}^{- \frac{1}{2}}_{s, t} (\theta_{s, *} - \hat{\mu}_{s, t})}\,.
\end{align*}
By definition, $\theta_{s, *} - \hat{\mu}_{s, t} \mid H_t \sim \cN(\mathbf{0}, \hat{\Sigma}_{s, t})$, and hence $\hat{\Sigma}^{- \frac{1}{2}}_{s, t} (\theta_{s, *} - \hat{\mu}_{s, t}) \mid H_t$ is a $d$-dimensional standard normal variable. Moreover, note that $\bar{E}_{s, t}$ implies $\maxnorm{\hat{\Sigma}^{- \frac{1}{2}}_{s, t} (\theta_{s, *} - \hat{\mu}_{s, t})} \geq \sqrt{2 \log(1 / \delta)}$. Finally, we combine these facts with a union bound over all entries of $\hat{\Sigma}^{- \frac{1}{2}}_{s, t} (\theta_{s, *} - \hat{\mu}_{s, t}) \mid H_t$, which are standard normal variables, and get
\begin{align*}
  \condE{\maxnorm{\hat{\Sigma}^{- \frac{1}{2}}_{s, t} (\theta_{s, *} - \hat{\mu}_{s, t})}
  \I{\bar{E}_{s, t}}}{H_t}
  \leq 2 \sum_{i = 1}^d \frac{1}{\sqrt{2 \pi}} \int_{u = \sqrt{2 \log(1 / \delta)}}^\infty
  u \exp\left[- \frac{u^2}{2}\right] \dif u
  \leq \sqrt{\frac{2}{\pi}} d \delta\,.
\end{align*}
Now we combine all inequalities and have
\begin{align*}
  \condE{A_{s, *}\T (\theta_{s, *} - \hat{\mu}_{s, t})}{H_t}
  \leq \sqrt{2 d \log(1 / \delta)} \, \condE{\normw{A_{s, t}}{\hat{\Sigma}_{s, t}}}{H_t} +
  \sqrt{\frac{2}{\pi}} \sigma_{\max} d^\frac{3}{2} \delta\,.
\end{align*}
Since the above bound holds for any history $H_t$, we combine everything and get
\begin{align*}
  \E{}{\sum_{t \geq 1} \sum_{s \in \cS_t} A_{s, *}\T \theta_{s, *} - A_{s, t}\T \theta_{s, *}}
  & \leq \sqrt{2 d \log(1 / \delta)} \,
  \E{}{\sum_{t \geq 1} \sum_{s \in \cS_t} \normw{A_{s, t}}{\hat{\Sigma}_{s, t}}} +
  \sqrt{\frac{2}{\pi}} \sigma_{\max} d^\frac{3}{2} m n \delta \\
  & \leq \sqrt{2 d m n \log(1 / \delta)}
  \sqrt{\E{}{\sum_{t \geq 1} \sum_{s \in \cS_t} \normw{A_{s, t}}{\hat{\Sigma}_{s, t}}^2}} +
  \sqrt{\frac{2}{\pi}} \sigma_{\max} d^\frac{3}{2} m n \delta\,.
\end{align*}
The last step uses the Cauchy-Schwarz inequality and the concavity of the square root.

To bound $\sigma_{\max}$, we use Weyl's inequalities together with \eqref{eq:covariance decomposition}, the second claim in \cref{lem:covariance decomposition}, and \eqref{eq:linear hyperposterior}. Specifically, under the assumption that $\normw{a}{2} \leq 1$ for all $a \in \cA$, we have
\begin{align*}
  \max_{a \in \cA} \normw{a}{\hat{\Sigma}_{s, t}}^2
  & \leq \lambda_1(\hat{\Sigma}_{s, t})
  \leq \lambda_1((\Sigma_0^{-1} + G_{s, t})^{-1}) +
  \lambda_1((\Sigma_0^{-1} + G_{s, t})^{-1} \Sigma_0^{-1} \bar{\Sigma}_t
  \Sigma_0^{-1} (\Sigma_0^{-1} + G_{s, t})^{-1}) \\
  & \leq \lambda_1(\Sigma_0) +
  \frac{\lambda_1^2(\Sigma_0) \lambda_1(\Sigma_q)}{\lambda_d^2(\Sigma_0)}
  = \sigma_{\max}^2\,.
\end{align*}
This concludes the proof of the first claim.

The second claim is proved by modifying the first proof as follows. Fix round $t$ and task $s \in \cS_t$. Let
\begin{align*}
  E_{s, t} =
  \set{\forall a \in \cA: |a\T (\theta_{s, *} - \hat{\mu}_{s, t})|
  \leq \sqrt{2 \log(1 / \delta)} \normw{a}{\hat{\Sigma}_{s, t}}}
\end{align*}
be the event that all high-probability confidence intervals hold. Then we have
\begin{align*}
  \condE{A_{s, *}\T (\theta_{s, *} - \hat{\mu}_{s, t})}{H_t}
  \leq \sqrt{2 \log(1 / \delta)} \, \condE{\normw{A_{s, t}}{\hat{\Sigma}_{s, t}}}{H_t} +
  \condE{A_{s, *}\T (\theta_{s, *} - \hat{\mu}_{s, t}) \I{\bar{E}_{s, t}}}{H_t}\,.
\end{align*}
Now note that for any action $a$, $a\T (\theta_{s, *} - \hat{\mu}_{s, t}) / \normw{a}{\hat{\Sigma}_{s, t}}$ is a standard normal variable. It follows that
\begin{align*}
  \condE{A_{s, *}\T (\theta_{s, *} - \hat{\mu}_{s, t}) \I{\bar{E}_{s, t}}}{H_t}
  \leq 2 \sum_{a \in \cA} \normw{a}{\hat{\Sigma}_{s, t}} \frac{1}{\sqrt{2 \pi}}
  \int_{u = \sqrt{2 \log(1 / \delta)}}^\infty
  u \exp\left[- \frac{u^2}{2}\right] \dif u
  \leq \sqrt{\frac{2}{\pi}} \sigma_{\max} K \delta\,.
\end{align*}
The rest of the proof proceeds as in the first claim, yielding
\begin{align*}
  \condE{A_{s, *}\T (\theta_{s, *} - \hat{\mu}_{s, t})}{H_t}
  \leq \sqrt{2 \log(1 / \delta)} \, \condE{\normw{A_{s, t}}{\hat{\Sigma}_{s, t}}}{H_t} +
  \sqrt{\frac{2}{\pi}} \sigma_{\max} K \delta\,.
\end{align*}
This completes the proof.

\section{Proof of \cref{thm:sequential regret}}
\label{sec:sequential proof}

\cref{lem:bayes regret} says that the Bayes regret $\Bregret(m, n)$ can be bounded by bounding the sum of posterior variances $\cV(m, n)$. Since $|\cS_t| = 1$, we make two simplifications. First, we replace the set of tasks $\cS_t$ by a single task $S_t \in [m]$. Second, there are exactly $m n$ rounds.

Fix round $t$ and task $s = S_t$. To reduce clutter, let $M = \Sigma_0^{-1} + G_{s, t}$. By the total covariance decomposition in \cref{lem:covariance decomposition}, we have that
\begin{align}
  \normw{A_{s, t}}{\hat{\Sigma}_{s, t}}^2
  & = \sigma^2 \frac{A_{s, t}\T \hat{\Sigma}_{s, t} A_{s, t}}{\sigma^2}
  = \sigma^2 \left(\sigma^{-2} A_{s, t}\T \tilde{\Sigma}_{s, t} A_{s, t} +
  \sigma^{-2} A_{s, t}\T M^{-1} \Sigma_0^{-1} \bar{\Sigma}_t
  \Sigma_0^{-1} M^{-1} A_{s, t}\right)
  \nonumber \\
  & \leq c_1 \log(1 + \sigma^{-2} A_{s, t}\T \tilde{\Sigma}_{s, t} A_{s, t}) +
  c_2 \log(1 + \sigma^{-2} A_{s, t}\T M^{-1} \Sigma_0^{-1} \bar{\Sigma}_t
  \Sigma_0^{-1} M^{-1} A_{s, t})
  \nonumber \\
  & = c_1 \log\det(I_d + \sigma^{-2}
  \tilde{\Sigma}_{s, t}^\frac{1}{2} A_{s, t} A_{s, t}\T \tilde{\Sigma}_{s, t}^\frac{1}{2}) +
  c_2 \log\det(I_d + \sigma^{-2}
  \bar{\Sigma}^\frac{1}{2}_t \Sigma_0^{-1} M^{-1} A_{s, t} A_{s, t}\T
  M^{-1} \Sigma_0^{-1} \bar{\Sigma}^\frac{1}{2}_t)\,.
  \label{eq:sequential proof decomposition}
\end{align}
The logarithmic terms are introduced using
\begin{align*}
  x
  = \frac{x}{\log(1 + x)} \log(1 + x)
  \leq \left(\max_{x \in [0, u]} \frac{x}{\log(1 + x)}\right) \log(1 + x)
  = \frac{u}{\log(1 + u)} \log(1 + x)\,,
\end{align*}
which holds for any $x \in [0, u]$. The resulting constants are
\begin{align*}
  c_1
  = \frac{\lambda_1(\Sigma_0)}{\log(1 + \sigma^{-2} \lambda_1(\Sigma_0))}\,, \quad
  c_2
  = \frac{c_q}{\log(1 + \sigma^{-2} c_q)}\,, \quad
  c_q
  = \frac{\lambda_1^2(\Sigma_0) \lambda_1(\Sigma_q)}{\lambda_d^2(\Sigma_0)}\,.
\end{align*}
The derivation of $c_1$ uses that
\begin{align*}
  A_{s, t}\T \tilde{\Sigma}_{s, t} A_{s, t}
  = \lambda_1(\tilde{\Sigma}_{s, t})
  = \lambda_d^{-1}(\Sigma_0^{-1} + G_{s, t})
  \leq \lambda_d^{-1}(\Sigma_0^{-1})
  = \lambda_1(\Sigma_0)\,.
\end{align*}
The derivation of $c_2$ follows from
\begin{align*}
  A_{s, t}\T M^{-1} \Sigma_0^{-1} \bar{\Sigma}_t \Sigma_0^{-1} M^{-1} A_{s, t}
  \leq \lambda_1^2(M^{-1}) \lambda_1^2(\Sigma_0^{-1}) \lambda_1(\bar{\Sigma}_t)
  \leq \frac{\lambda_1^2(\Sigma_0) \lambda_1(\Sigma_q)}{\lambda_d^2(\Sigma_0)}\,.
\end{align*}
This is also proved as the second claim in \cref{lem:covariance decomposition}. Now we focus on bounding the logarithmic terms in \eqref{eq:sequential proof decomposition}.

\subsection{First Term in \eqref{eq:sequential proof decomposition}}
\label{sec:sequential proof 1}

This is a per-instance term and can be rewritten as
\begin{align*}
  \log\det(I_d + \sigma^{-2}
  \tilde{\Sigma}_{s, t}^\frac{1}{2} A_{s, t} A_{s, t}\T \tilde{\Sigma}_{s, t}^\frac{1}{2})
  = \log\det(\tilde{\Sigma}_{s, t}^{-1} + \sigma^{-2} A_{s, t} A_{s, t}\T) - \log\det(\tilde{\Sigma}_{s, t}^{-1})\,.
\end{align*}
When we sum over all rounds with task $s$, we get telescoping and the contribution of this term is at most
\begin{align*}
  \sum_{t = 1}^{m n} \I{S_t = s} 
  \log\det(I_d + \sigma^{-2} \tilde{\Sigma}_{s, t}^\frac{1}{2} A_{s, t}
  A_{s, t}\T \tilde{\Sigma}_{s, t}^\frac{1}{2})
  & = \log\det(\tilde{\Sigma}_{s, m n + 1}^{-1}) - \log\det(\tilde{\Sigma}_{s, 1}^{-1})
  = \log\det(\Sigma_0^\frac{1}{2} \tilde{\Sigma}_{s, m n + 1}^{-1} \Sigma_0^\frac{1}{2}) \\
  & \leq d \log\left(\frac{1}{d} \trace(\Sigma_0^\frac{1}{2} \tilde{\Sigma}_{s, m n + 1}^{-1}
  \Sigma_0^\frac{1}{2})\right)
  \leq d \log\left(1 + \frac{\lambda_1(\Sigma_0) n}{\sigma^2 d}\right)\,,
\end{align*}
where we use that task $s$ appears at most $n$ times. Now we sum over all $m$ tasks and get
\begin{align*}
  \sum_{t = 1}^{m n}
  \log\det(I_d + \sigma^{-2} \tilde{\Sigma}_{S_t, t}^\frac{1}{2} A_{S_t, t}
  A_{S_t, t}\T \tilde{\Sigma}_{S_t, t}^\frac{1}{2})
  \leq d m \log\left(1 + \frac{\lambda_1(\Sigma_0) n}{\sigma^2 d}\right)\,.
\end{align*}

\subsection{Second Term in \eqref{eq:sequential proof decomposition}}
\label{sec:sequential proof 2}

This is a hyper-parameter term. Before we analyze it, let $v = \sigma^{-1} M^{- \frac{1}{2}} A_{s, t}$ and note that
\begin{align}
  \bar{\Sigma}_{t + 1}^{-1} - \bar{\Sigma}_t^{-1}
  & = (\Sigma_0 + (G_{s, t} + \sigma^{-2} A_{s, t} A_{s, t}\T)^{-1})^{-1} -
  (\Sigma_0 + G_{s, t}^{-1})^{-1}
  \nonumber \\
  & = \Sigma_0^{-1} - \Sigma_0^{-1} (M + \sigma^{-2} A_{s, t} A_{s, t}\T)^{-1} \Sigma_0^{-1} -
  (\Sigma_0^{-1} - \Sigma_0^{-1} M^{-1} \Sigma_0^{-1})
  \nonumber \\
  & = \Sigma_0^{-1} (M^{-1} - (M + \sigma^{-2} A_{s, t} A_{s, t}\T)^{-1}) \Sigma_0^{-1}
  \nonumber \\
  & = \Sigma_0^{-1} M^{- \frac{1}{2}}
  (I_d - (I_d + \sigma^{-2} M^{- \frac{1}{2}} A_{s, t} A_{s, t}\T M^{- \frac{1}{2}})^{-1})
  M^{- \frac{1}{2}} \Sigma_0^{-1}
  \nonumber \\
  & = \Sigma_0^{-1} M^{- \frac{1}{2}}
  (I_d - (I_d + v v\T)^{-1})
  M^{- \frac{1}{2}} \Sigma_0^{-1}
  \nonumber \\
  & = \Sigma_0^{-1} M^{- \frac{1}{2}}
  \frac{v v\T}{1 + v\T v}
  M^{- \frac{1}{2}} \Sigma_0^{-1}
  \nonumber \\
  & = \sigma^{-2} \Sigma_0^{-1} M^{-1}
  \frac{A_{s, t} A_{s, t}\T}{1 + v\T v}
  M^{-1} \Sigma_0^{-1}\,,
  \label{eq:linear telescoping}
\end{align}
where we first use the Woodbury matrix identity and then the Sherman-Morrison formula. Since $\normw{A_{s, t}}{2} \leq 1$,
\begin{align*}
  1 + v\T v
  = 1 + \sigma^{-2} A_{s, t}\T M^{-1} A_{s, t}
  \leq 1 + \sigma^{-2} \lambda_1(\Sigma_0) = c\,.
\end{align*}
Based on the above derivations, we bound the second logarithmic term in \eqref{eq:sequential proof decomposition} as
\begin{align*}
  & \log\det(I_d +
  \sigma^{-2} \bar{\Sigma}_t^\frac{1}{2} \Sigma_0^{-1} M^{-1} A_{s, t} A_{s, t}\T
  M^{-1} \Sigma_0^{-1} \bar{\Sigma}_t^\frac{1}{2}) \\
  & \quad \leq c \log\det(I_d +
  \sigma^{-2} \bar{\Sigma}^\frac{1}{2}_t \Sigma_0^{-1} M^{-1} A_{s, t} A_{s, t}\T
  M^{-1} \Sigma_0^{-1} \bar{\Sigma}^\frac{1}{2}_t / c) \\
  & \quad = c \left[\log\det(\bar{\Sigma}_t^{-1} +
  \sigma^{-2} \Sigma_0^{-1} M^{-1} A_{s, t} A_{s, t}\T M^{-1} \Sigma_0^{-1} / c) -
  \log\det(\bar{\Sigma}_t^{-1})\right] \\
  & \quad \leq c \left[\log\det(\bar{\Sigma}_{t + 1}^{-1}) -
  \log\det(\bar{\Sigma}_t^{-1})\right]\,.
\end{align*}
The first inequality holds because $\log(1 + x) \leq c \log(1 + x / c)$ for any $x \geq 0$ and $c \geq 1$. The second inequality follows from the fact that we have a rank-$1$ update of $\bar{\Sigma}_t^{-1}$. Now we sum over all rounds and get telescoping
\begin{align*}
  & \sum_{t = 1}^{m n}
  \log\det(I_d + \sigma^{-2} \bar{\Sigma}_t^\frac{1}{2} \Sigma_0^{-1}
  (\Sigma_0^{-1} + G_{S_t, t})^{-1} A_{S_t, t} A_{S_t, t}\T (\Sigma_0^{-1} + G_{S_t, t})^{-1}
  \Sigma_0^{-1} \bar{\Sigma}_t^\frac{1}{2}) \\
  & \quad \leq c \left[\log\det(\bar{\Sigma}_{m n + 1}^{-1}) -
  \log\det(\bar{\Sigma}_1^{-1})\right]
  = c \log\det(\Sigma_q^\frac{1}{2} \bar{\Sigma}_{m n + 1}^{-1} \Sigma_q^\frac{1}{2})
  \leq c d \log\left(\frac{1}{d} \trace(\Sigma_q^\frac{1}{2} \bar{\Sigma}_{m n + 1}^{-1}
  \Sigma_q^\frac{1}{2})\right) \\
  & \quad \leq c d
  \log(\lambda_1(\Sigma_q^\frac{1}{2} \bar{\Sigma}_{m n + 1}^{-1} \Sigma_q^\frac{1}{2}))
  \leq c d \log\left(1 + \frac{\lambda_1(\Sigma_q) m}{\lambda_d(\Sigma_0)}\right)\,.
\end{align*}
Finally, we combine the upper bounds for both logarithmic terms and get
\begin{align*}
  \cV(m, n)
  = \E{}{\sum_{t = 1}^{m n} \normw{A_{S_t, t}}{\hat{\Sigma}_{S_t, t}}^2}
  \leq d \left[c_1 m \log\left(1 + \frac{\lambda_1(\Sigma_0) n}{\sigma^2 d}\right) +
  c_2 c \log\left(1 + \frac{\lambda_1(\Sigma_q) m}{\lambda_d(\Sigma_0)}\right)\right]\,,
\end{align*}
which yields the desired result after we substitute this bound into \cref{lem:bayes regret}. To simplify presentation in the main paper, $c_1$ and $c_2$ in \cref{thm:sequential regret} include the above logarithmic terms that multiply them.

\section{Proof of \cref{thm:concurrent regret}}
\label{sec:concurrent proof}

From \cref{ass:basis}, there exists a basis of $d$ actions such that if all actions in the basis are taken in task $s$ by round $t$, it is guaranteed that $\lambda_d(G_{s, t}) \geq \eta / \sigma^2$. We modify \hierts to takes these actions first in any task $s$. Let $\cC_t = \{s \in \cS_t: \lambda_d(G_{s, t}) \geq \eta / \sigma^2\}$ be the set of \emph{sufficiently-explored tasks} by round $t$.

Using $\cC_t$, we decompose the Bayes regret as
\begin{align*}
  \Bregret(m, n)
  \leq 
  \E{}{\sum_{t \geq 1} \sum_{s \in \cS_t} \I{s \in \cC_t}
  (A_{s, *}\T \theta_{s, *} - A_{s, t}\T \theta_{s, *})} + 
  \E{}{\sum_{t \geq 1} \sum_{s \in \cS_t} \I{s \not\in \cC_t}
  (A_{s, *}\T \theta_{s, *} - A_{s, t}\T \theta_{s, *})}\,.
\end{align*}
For any task $s$ and round $t$, we can trivially bound
\begin{align*}
  \E{}{(A_{s, *} - A_{s, t})\T \theta_{s, *}}
  \leq \E{}{\normw{A_{s, *} - A_{s, t}}{\hat{\Sigma}_{s, 1}}
  \normw{\theta_{s, *}}{\hat{\Sigma}_{s, 1}^{-1}}}
  \leq 2 \sigma_{\max} \left(\normw{\mu_q}{\hat{\Sigma}_{s, 1}^{-1}} +
  \E{}{\normw{\theta_{s, *} - \mu_q}{\hat{\Sigma}_{s, 1}^{-1}}}\right)\,,
\end{align*}
where $\sigma_{\max} = \sqrt{\lambda_1(\Sigma_q + \Sigma_0)}$ as in \cref{sec:sequential proof}. Here we use that $\normw{A_{s, *} - A_{s, t}}{2} \leq 2$ and that the prior covariance of $\theta_{s, *}$ is $\hat{\Sigma}_{s, 1} = \Sigma_q + \Sigma_0$. We know from \eqref{eq:gaussian hierarchical} that $\theta_{s, *} - \mu_q \sim \cN(\mathbf{0}, \Sigma_q + \Sigma_0)$. This means that $\hat{\Sigma}^{- \frac{1}{2}}_{s, 1} (\theta_{s, *} - \mu_q)$ is a vector of $d$ independent standard normal variables. It follows that
\begin{align*}
  \E{}{\normw{\theta_{s, *} - \mu_q}{\hat{\Sigma}_{s, 1}^{-1}}}
  = \E{}{\normw{\hat{\Sigma}^{- \frac{1}{2}}_{s, 1} (\theta_{s, *} - \mu_q)}{2}}
  \leq \sqrt{\E{}{\normw{\hat{\Sigma}^{- \frac{1}{2}}_{s, 1}
  (\theta_{s, *} - \mu_q)}{2}^2}}
  = \sqrt{d}\,.
\end{align*}
Since $s \not\in \cC_t$ occurs at most $d$ times for any task $s$, the total regret due to forced exploration is bounded as
\begin{align*}
  \E{}{\sum_{t \geq 1} \sum_{s \in \cS_t} \I{s \not\in \cC_t}
  (A_{s, *}\T \theta_{s, *} - A_{s, t}\T \theta_{s, *})}
  \leq 2 \sigma_{\max} \left(\normw{\mu_q}{\hat{\Sigma}_{s, 1}^{-1}} + \sqrt{d}\right) d m
  = c_3\,.
\end{align*}
It remains to bound the first term in $\Bregret(m, n)$. On event $s \in \cC_t$, \hierts samples from the posterior and behaves exactly as \cref{alg:ts}. Therefore, we only need to bound $\cV(m, n) = \E{}{\sum_{t \geq 1} \sum_{s \in \cS_t} \I{s \in \cC_t} \normw{A_{s, t}}{\hat{\Sigma}_{s, t}}^2}$ and then substitute the bound into \cref{lem:bayes regret}. By the total covariance decomposition in \cref{lem:covariance decomposition}, we have
\begin{align}
  \normw{A_{s, t}}{\hat{\Sigma}_{s, t}}^2
  = A_{s, t}\T \tilde{\Sigma}_{s, t} A_{s, t} +
  A_{s, t}\T M^{-1} \Sigma_0^{-1} \bar{\Sigma}_t \Sigma_0^{-1} M^{-1} A_{s, t}\,,
  \label{eq:concurrent proof decomposition}
\end{align}
where $M = \Sigma_0^{-1} + G_{s, t}$ to reduce clutter. As in \cref{sec:sequential proof}, we bound the contribution of each term separately.

\subsection{First Term in \eqref{eq:concurrent proof decomposition}}

This term depends only on $\tilde{\Sigma}_{s, t}$, which does not depend on interactions with other tasks than task $s$. Therefore, the bound is the same as in the sequential case in \cref{sec:sequential proof 1},
\begin{align*}
  \sum_{t \geq 1} \sum_{s \in \cS_t} \I{s \in \cC_t} A_{s, t}\T \tilde{\Sigma}_{s, t} A_{s, t}
  \leq c_1 d m \log\left(1 + \frac{\lambda_1(\Sigma_0) n}{\sigma^2 d}\right)\,,
\end{align*}
where $c_1$ is defined in \cref{sec:sequential proof}.

\subsection{Second Term in \eqref{eq:concurrent proof decomposition}}

The difference from the sequential setting is in how we bound the second term in \eqref{eq:concurrent proof decomposition}. Before we had $|\cS_t| = 1$, while now we have $|\cS_t| \leq L \leq m$ for some $L$. Since more than one task is acted upon per round, the telescoping identity in \eqref{eq:linear telescoping} no longer holds. To remedy this, we reduce the concurrent case to the sequential one. Specifically, suppose that task $s \in \cS_t$ in round $t$ has access to the concurrent observations of prior tasks in round $t$, for some order of tasks $\cS_t = \{S_{t, i}\}_{i = 1}^L$. As \cref{thm:sequential regret} holds for any order, we choose the order where sufficiently-explored tasks $s \in \cC_t$ appear first.

Let $\cS_{t, i} = \{S_{t, j}\}_{j = 1}^{i - 1}$ be the first $i - 1$ tasks in $\cS_t$ according to our chosen order. For $s = S_{t, i}$, let
\begin{align*}
  \bar{\Sigma}_{s, t}^{-1}
  = \Sigma_q^{-1} + \sum_{z \in \cS_{t, i}} (\Sigma_0 + G_{z, t + 1}^{-1})^{-1} +
  \sum_{z \in [m] \setminus \cS_{t, i}} (\Sigma_0 + G_{z, t}^{-1})^{-1}
\end{align*}
be the reciprocal of the hyper-posterior covariance updated with concurrent observations in tasks $\cS_{t, i}$. Next we show that $\bar{\Sigma}_t$ and $\bar{\Sigma}_{s, t}$ are similar.

\begin{lemma}
\label{lem:sequential concurrent ratio} Fix round $t$ and $i \in [L]$. Let $s = S_{t, i}$ and $\lambda_d(G_{s, t}) \geq \eta / \sigma^2$. Then
\begin{align*}
  \lambda_1(\bar{\Sigma}_{s, t}^{-1} \bar{\Sigma}_t)
  \leq 1 + \frac{\sigma^{-2} \lambda_1(\Sigma_q) (\lambda_1(\Sigma_0) + \sigma^2 / \eta)}
  {\lambda_1(\Sigma_q) + (\lambda_1(\Sigma_0) + \sigma^2 / \eta) / L}\,.
\end{align*}
\end{lemma}
\begin{proof}
Using standard eigenvalue inequalities, we have
\begin{align}
  \lambda_1(\bar{\Sigma}_{s, t}^{-1} \bar{\Sigma}_t)
  = \lambda_1((\bar{\Sigma}_t^{-1} + \bar{\Sigma}_{s, t}^{-1} - \bar{\Sigma}_t^{-1})
  \bar{\Sigma}_t)
  \leq 1 + \lambda_1((\bar{\Sigma}_{s, t}^{-1} - \bar{\Sigma}_t^{-1}) \bar{\Sigma}_t)
  \leq 1 + \frac{\lambda_1(\bar{\Sigma}_{s, t}^{-1} - \bar{\Sigma}_t^{-1})}
  {\lambda_d(\bar{\Sigma}_t^{-1})}\,.
  \label{eq:ratio decomposition}
\end{align}
By Weyl's inequalities, and from the definition of $\bar{\Sigma}_t$, we have
\begin{align*}
  \lambda_d(\bar{\Sigma}_t^{-1}) 
  & \geq \lambda_d(\Sigma_q^{-1}) + \sum_{z \in [m]} \lambda_d((\Sigma_0 + G_{z, t}^{-1})^{-1})
  = \lambda_d(\Sigma_q^{-1}) + \sum_{z \in [m]} \lambda_1^{-1}(\Sigma_0 + G_{z, t}^{-1}) \\
  & \geq \lambda_d(\Sigma_q^{-1}) +
  \sum_{z \in [m]} (\lambda_1(\Sigma_0) + \lambda_1(G_{z, t}^{-1}))^{-1}
  \geq \lambda_d(\Sigma_q^{-1}) + (i - 1) (\lambda_1(\Sigma_0) + \sigma^2 / \eta)^{-1}\,.
\end{align*}
In the last inequality, we use that the previous $i - 1$ tasks $\cS_{t, i}$ are sufficiently explored. Analogously to \eqref{eq:linear telescoping},
\begin{align*}
  \bar{\Sigma}_{s, t}^{-1} - \bar{\Sigma}_t^{-1}
  & = \sum_{z \in \cS_{t, i}}
  (\Sigma_0 + (G_{z, t} + \sigma^{-2} A_{z, t} A_{z, t}\T)^{-1})^{-1} -
  (\Sigma_0 + G_{z, t}^{-1})^{-1} \\
  & = \sigma^{-2} \sum_{z \in \cS_{t, i}} \Sigma_0^{-1} M_{z, t}^{-1}
  \frac{A_{z, t} A_{z, t}\T}{1 + \sigma^{-2} A_{z, t}\T M_{z, t}^{-1} A_{z, t}}
  M_{z, t}^{-1} \Sigma_0^{-1}\,,
\end{align*} 
where $M_{z, t} = \Sigma_0^{-1} + G_{z, t}$ to reduce clutter. Moreover, since $\normw{A_{z, t}}{2} \leq 1$ and $\sigma^{-2} A_{z, t}\T M_{z, t}^{-1} A_{z, t} \geq 0$, we have
\begin{align*}
  \lambda_1(\bar{\Sigma}_{s, t}^{-1} - \bar{\Sigma}_t^{-1})
  \leq (i - 1) \sigma^{-2}\,.
\end{align*}
Finally, we substitute our upper bounds to the right-hand side of \eqref{eq:ratio decomposition} and get
\begin{align*}
  \frac{\lambda_1(\bar{\Sigma}_{s, t}^{-1} - \bar{\Sigma}_t^{-1})}
  {\lambda_d(\bar{\Sigma}_t^{-1})} \leq
  \frac{(i - 1) \sigma^{-2}}{\lambda_1^{-1}(\Sigma_q) +
  (i - 1) (\lambda_1(\Sigma_0) + \sigma^2 / \eta)^{-1}}
  \leq \frac{\sigma^{-2} \lambda_1(\Sigma_q) (\lambda_1(\Sigma_0) + \sigma^2 / \eta)}
  {\lambda_1(\Sigma_q) + (\lambda_1(\Sigma_0) + \sigma^2 / \eta) / L}\,,
\end{align*}
where we use that the ratio is maximized when $i - 1 = L$. This completes the proof.
\end{proof}

Now we return to \eqref{eq:concurrent proof decomposition}. First, we have that
\begin{align*}
  A_{s, t}\T M^{-1} \Sigma_0^{-1} \bar{\Sigma}_t \Sigma_0^{-1} M^{-1} A_{s, t}
  & = A_{s, t}\T M^{-1} \Sigma_0^{-1} \bar{\Sigma}_{s, t}^\frac{1}{2}
  \left(\bar{\Sigma}_{s, t}^{- \frac{1}{2}} \bar{\Sigma}_t^\frac{1}{2}
  \bar{\Sigma}_t^\frac{1}{2} \bar{\Sigma}_{s, t}^{- \frac{1}{2}}\right)
  \bar{\Sigma}_{s, t}^\frac{1}{2} \Sigma_0^{-1} M^{-1} A_{s, t} \\
  & \leq \lambda_1(\bar{\Sigma}_{s, t}^{- \frac{1}{2}} \bar{\Sigma}_t^\frac{1}{2}
  \bar{\Sigma}_t^\frac{1}{2} \bar{\Sigma}_{s, t}^{- \frac{1}{2}})
  A_{s, t}\T M^{-1} \Sigma_0^{-1} \bar{\Sigma}_{s, t} \Sigma_0^{-1} M^{-1} A_{s, t} \\
  & \leq \lambda_1(\bar{\Sigma}_{s, t}^{-1} \bar{\Sigma}_t)
  A_{s, t}\T M^{-1} \Sigma_0^{-1} \bar{\Sigma}_{s, t} \Sigma_0^{-1} M^{-1} A_{s, t}\,,
\end{align*}
where we use that the above expression is a quadratic form. Next we apply \cref{lem:sequential concurrent ratio} and get
\begin{align*}
  \lambda_1(\bar{\Sigma}_{s, t}^{-1} \bar{\Sigma}_t)
  \leq 1 + \frac{\sigma^{-2} \lambda_1(\Sigma_q) (\lambda_1(\Sigma_0) + \sigma^2 / \eta)}
  {\lambda_1(\Sigma_q) + (\lambda_1(\Sigma_0) + \sigma^2 / \eta) / L}
  = c_4\,.
\end{align*}
After $\bar{\Sigma}_t$ is turned into $\bar{\Sigma}_{s, t}$, we follow \cref{sec:sequential proof 2} and get that the hyper-parameter regret is
\begin{align*}
  c_2 c_4 c d \log\left(1 + \frac{\lambda_1(\Sigma_q) m}{\lambda_d(\Sigma_0)}\right)\,,
\end{align*}
where the only difference is the extra factor of $c_4$. Finally, we combine all upper bounds and get
\begin{align*}
  \cV(m, n)
  = \E{}{\sum_{t \geq 1} \sum_{s \in \cS_t} \I{s \in \cC_t}
  \normw{A_{s, t}}{\hat{\Sigma}_{s, t}}^2}
  \leq d \left[c_1 m \log\left(1 + \frac{\lambda_1(\Sigma_0) n}{\sigma^2 d}\right) +
  c_2 c_4 c \log\left(1 + \frac{\lambda_1(\Sigma_q) m}{\lambda_d(\Sigma_0)}\right)\right]\,,
\end{align*}
which yields the desired result after we substitute it into \cref{lem:bayes regret}. To simplify presentation in the main paper, $c_1$ and $c_2$ in \cref{thm:concurrent regret} include the above logarithmic terms that multiply them.

\section{Gaussian Bandit Regret Bounds}
\label{sec:mab bounds}

Our regret bounds in \cref{sec:regret bounds} can be specialized to $K$-armed Gaussian bandits (\cref{sec:gaussian bandit}). Specifically, when the action set $\cA = \set{e_i}_{i \in [K]}$ is the standard Euclidean basis in $\realset^K$, \cref{thm:sequential regret,thm:concurrent regret} can be restated as follows.

\begin{theorem}[Sequential Gaussian bandit regret]
\label{thm:sequential mab regret} Let $|\cS_t| = 1$ for all rounds $t$. Let $\delta = 1 / (m n)$. Then the Bayes regret of \hierts is
\begin{align*}
  \Bregret(m, n)
  \leq \sqrt{2 K m n [c_1 m + c_2] \log(m n)} + c_3\,,
\end{align*}
where $c_3 = O(K)$,
\begin{align*}
  c_1
  = \frac{\sigma_0^2}{\log(1 + \sigma^{-2} \sigma_0^2)}
  \log\left(1 + \frac{\sigma_0^2 n}{\sigma^2 K}\right)\,, \quad
  c_2
  = \frac{\sigma_q^2 c}{\log(1 + \sigma^{-2} \sigma_q^2)}
  \log\left(1 + \frac{\sigma_q^2 m}{\sigma_0^2}\right)\,, \quad
  c 
  = 1 + \frac{\sigma_0^2}{\sigma^2}\,.
\end{align*}
\end{theorem}

The main difference from the proof of \cref{thm:sequential regret} is that we start with the finite-action bound in \cref{lem:bayes regret}. Other than that, we use the facts that $\lambda_1(\Sigma_0) = \lambda_d(\Sigma_0) = \sigma_0^2$ and $\lambda_1(\Sigma_q) = \sigma_q^2$.

\begin{theorem}[Concurrent Gaussian bandit regret]
\label{thm:concurrent mab regret} Let $|\cS_t| \leq L \leq m$. Let $\delta = 1 / (m n)$. Then the Bayes regret of \hierts is
\begin{align*}
  \Bregret(m, n)
  \leq \sqrt{2 K m n [c_1 m + c_2] \log(m n)} + c_3\,,
\end{align*}
where $c_1$ and $c$ are defined as in \cref{thm:sequential mab regret},
\begin{align*}
  c_2
  = \frac{\sigma_q^2 c_4 c}{\log(1 + \sigma^{-2} \sigma_q^2)}
  \log\left(1 + \frac{\sigma_q^2 m}{\sigma_0^2}\right)\,, \quad
  c_4
  = 1 + \frac{\sigma^{-2} \sigma_q^2 (\sigma_0^2 + \sigma^2)}
  {\sigma_q^2 + (\sigma_0^2 + \sigma^2) / L}\,,
\end{align*}
and $c_3 = O(K m)$.
\end{theorem}

When we specialize \cref{thm:concurrent regret}, we note that $\eta = 1$, since the action set $\cA$ is the standard Euclidean basis.

\section{Image Classification Experiment}
\label{sec:classification experiments}

\begin{figure*}[t]
  \centering
  \begin{minipage}{0.45\textwidth}
    \includegraphics[width=\linewidth]{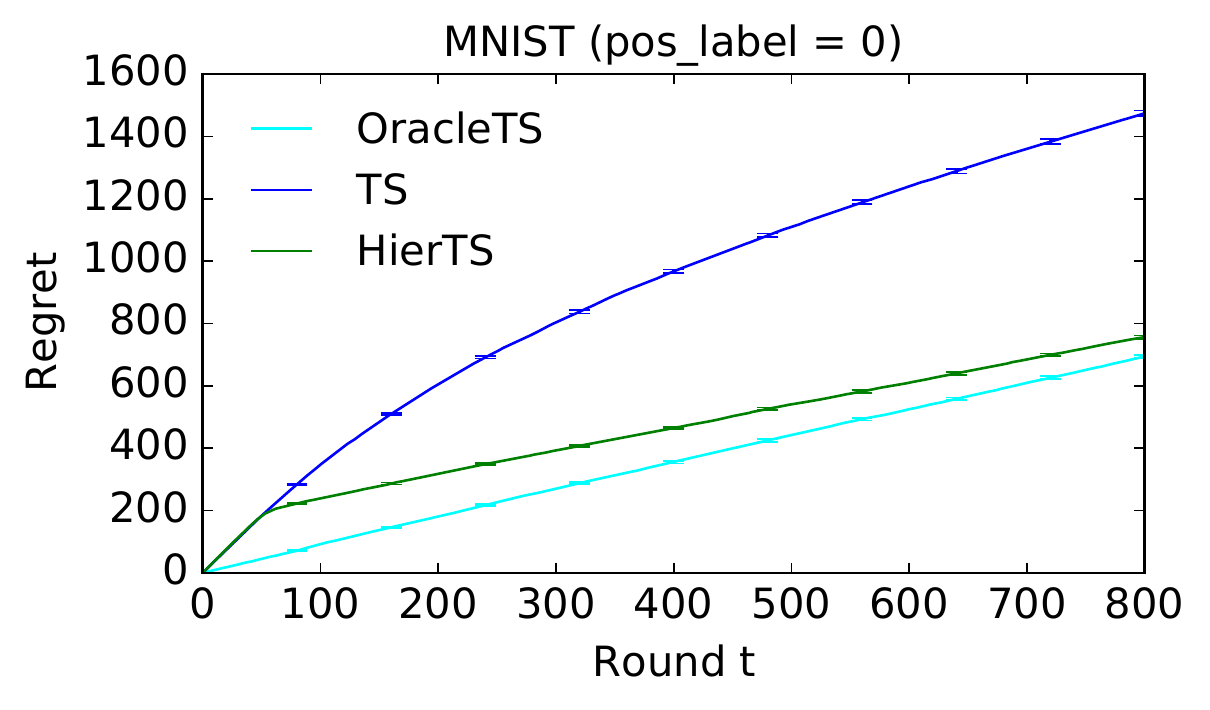}
  \end{minipage}
  \begin{minipage}{0.45\textwidth}
    \includegraphics[width=\linewidth]{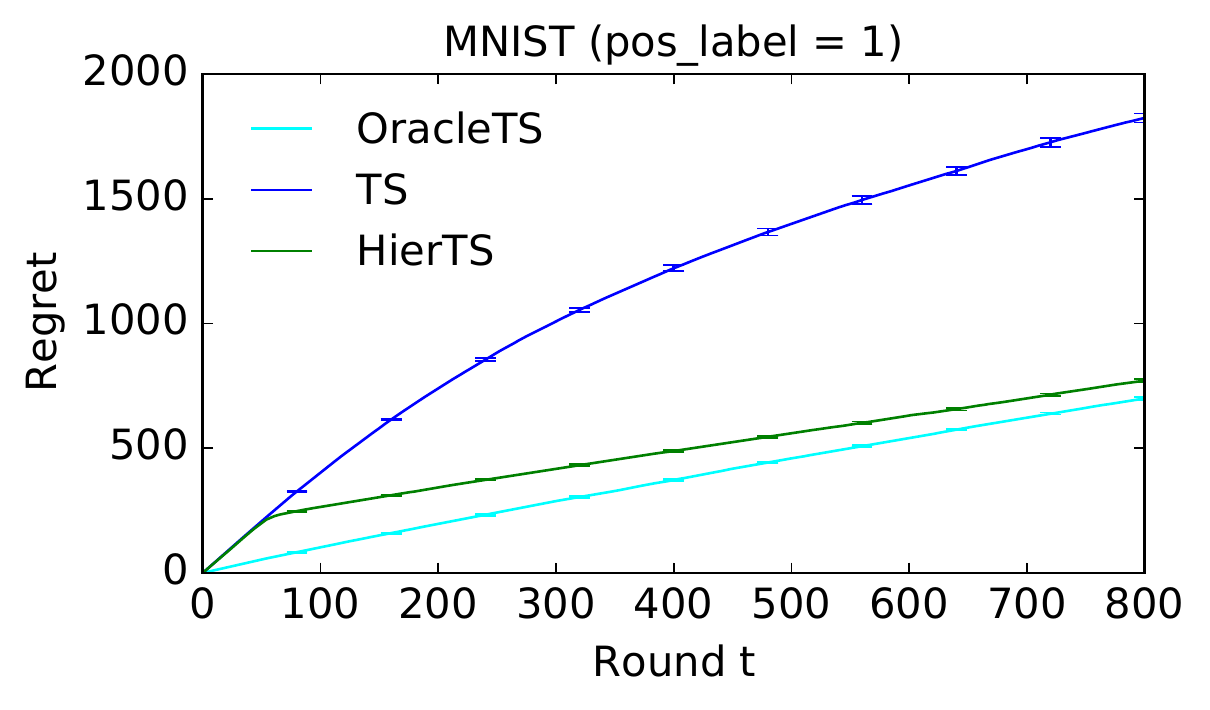}
  \end{minipage}
  \begin{minipage}{0.45\textwidth}
    \includegraphics[width=\linewidth]{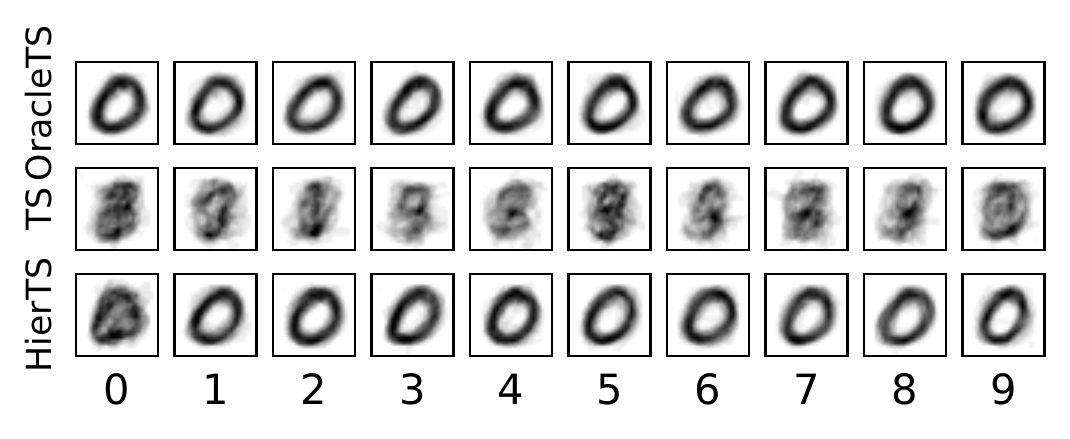}
  \end{minipage}
  \begin{minipage}{0.45\textwidth}
    \includegraphics[width=\linewidth]{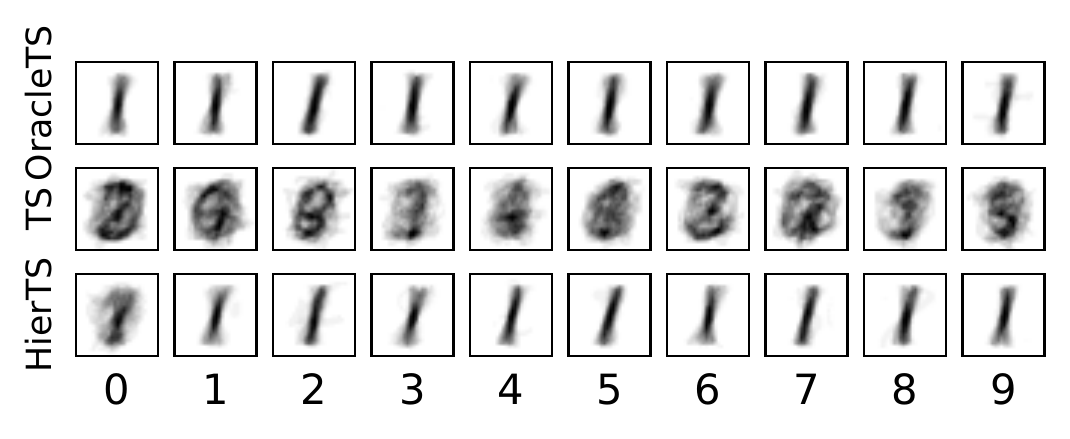}
  \end{minipage}
  \caption{Evaluation of \hierts on multi-task digit classification using MNIST with different positive image classes. On the top, we plot the cumulative Bayes regret at each round. On the bottom, we visualize the most-rewarding image according to the learned hyper-parameter at evenly-spaced intervals.}
  \label{fig:mnist}
\end{figure*}

We conduct an additional experiment that considers online classification using a real-world image dataset. The problem is cast as a multi-task linear bandit with Bernoulli rewards. Specifically, we construct a set of tasks where one image class is selected randomly to have high reward. In each task, at every round, $K$ images are uniformly sampled at random as actions, and the aim of the learning agent is to select an image from the unknown positive image class. The reward of an image from the positive class is $\mathsf{Ber}(0.9)$ and for all other classes is $\mathsf{Ber}(0.1)$.

We use the MNIST dataset \citep{mnist}, which consists of $60, 000$ images of handwritten digits, which we split equally into a training and test set. We down-sample each image to $d = 49$ pixels, which become the feature vector for the corresponding action in the bandit problem. 
For each digit, the training set is used to estimate $\mu_*, \Sigma_0$, where all three algorithms use $\Sigma_0$ but only \oraclets can use $\mu_*$. The algorithms are evaluated on the test set. 
Given a positive digit class, we construct a different task $s$ by sub-sampling from the test set, and computing $\theta_{s, *}$ using positive images from the sub-sampled data.
For each digit as the positive image class, we evaluate our three algorithms on a multi-task linear bandit with $m = 10$ tasks, $n = 400$ interactions per task, and $K = 30$ actions, uniformly sampled from the test images. We chose $L = 5$ tasks per round, leading to $800$ rounds total. We assume a hyper-prior of $Q = \cN(\mathbf{0}, I_d)$ with reward noise $\sigma = 0.5$ because the rewards are Bernoulli. 

\cref{fig:mnist} shows the performance of all algorithms for two digits across $20$ independent runs. We see that \hierts performs very well compared to standard \ts. In addition to regret, we also visualized the learned hyper-parameter $\bar{\mu}_t$ every $80$ rounds. We see that \hierts very quickly learns the correct hyper-parameter, showing that it effectively leverages the shared structure across task. Overall, this experiment shows that even if \hierts assumes a misspecified model of the environment, with non-Gaussian rewards and not knowing the true hyper-prior $Q$ and covariance $\Sigma_0$, \hierts still performs very well.

\end{document}